\documentclass[manuscript,natbib=false]{acmart}
\setcopyright{none}
\renewcommand\footnotetextcopyrightpermission[1]{}
\RequirePackage[
  datamodel=acmdatamodel,
  style=acmauthoryear,
  backend=biber,
  giveninits=true,
  uniquename=init
  ]{biblatex}

\usepackage{enumitem}
\usepackage{float}

\newcommand{\nSeeds}{30}
\newcommand{\PSfrontierSynLo}{-0.0055}
\newcommand{\PSfrontierSynHi}{+0.0015}
\newcommand{\PSfrontierSynSEmax}{0.0019}
\newcommand{\PSfrontierSynBelowCount}{3}
\newcommand{\synPostMatchedLo}{+0.0016}
\newcommand{\synPostMatchedHi}{+0.0032}
\newcommand{\synPostMatchedLowerMin}{0}
\newcommand{\synPostMatchedLowerMax}{3}
\newcommand{\synThrMatchedLo}{-0.0058}
\newcommand{\synThrMatchedHi}{-0.0031}
\newcommand{\synThrMatchedLowerMin}{30}
\newcommand{\synThrMatchedLowerMax}{30}
\newcommand{\synPostNoGateLo}{-0.0081}
\newcommand{\synPostNoGateHi}{-0.0026}
\newcommand{\synPostNoGateLowerMin}{23}
\newcommand{\synPostNoGateLowerMax}{30}
\newcommand{\synThrNoGateLo}{+0.0026}
\newcommand{\synThrNoGateHi}{+0.0162}

\newcommand{\synPostDtACILo}{+0.0054}
\newcommand{\synPostDtACIHi}{+0.0158}

\newcommand{\synPSPost}{0.090}

\newcommand{\synPSThr}{0.760}
\newcommand{\synMatchedPost}{0.088}

\newcommand{\synDtACIPost}{0.079}
\newcommand{\synDtACICov}{0.909}

\newcommand{\synKernelUpdPost}{0.098}

\newcommand{\synKernelUpdThr}{0.734}
\newcommand{\synNoUpdPost}{0.214}
\newcommand{\synNoUpdCov}{0.879}

\newcommand{\synRollingCov}{0.854}

\newcommand{\synStaticCov}{0.334}

\newcommand{\synGateRecPost}{0.088}

\newcommand{\synGateRecThr}{0.764}
\newcommand{\PSfrontierMaxAbs}{0.0055}

\newcommand{\regSynRank}{3.9}

\newcommand{\regMmluCatRank}{3.2}

\newcommand{\regMmluSubjRank}{3.2}

\newcommand{\regRtpRank}{3.0}

\newcommand{\regTqaRank}{3.9}

\newcommand{\lofosynthetic}{-0.0030}

\newcommand{\lofoMMLUcategories}{-0.0019}

\newcommand{\lofoMMLUsubjects}{-0.0006}

\newcommand{\lofoRTPbands}{-0.0048}

\newcommand{\lofoTruthfulQAcategories}{+0.0014}
\newcommand{\lofoTruthfulQAcategoriesSE}{0.0006}

\newcommand{\MmluCatStaticPost}{0.097}

\newcommand{\MmluCatMatchedPost}{0.099}

\newcommand{\MmluCatPSPost}{0.098}

\newcommand{\MmluCatPSvsMatchedPost}{-0.0010}
\newcommand{\MmluCatPSvsMatchedPostSE}{0.0007}

\newcommand{\MmluCatFrontierPS}{-0.0007}
\newcommand{\MmluCatFrontierPSSE}{0.0006}
\newcommand{\MmluCatFrontierPSIn}{+0.0001}
\newcommand{\MmluCatFrontierPSInSE}{0.0006}
\newcommand{\MmluCatFrontierPSInN}{17}
\newcommand{\MmluCatFrontierPSAllN}{24}
\newcommand{\MmluCatFrontierPSOutN}{7}
\newcommand{\MmluCatFrontierDtACI}{-0.0084}
\newcommand{\MmluCatFrontierDtACISE}{0.0012}

\newcommand{\MmluSubjFrontierPS}{-0.0001}

\newcommand{\MmluSubjFrontierDtACI}{-0.0060}

\newcommand{\RtpStaticCov}{0.867}

\newcommand{\RtpStaticPost}{0.144}

\newcommand{\RtpRollingPost}{0.114}

\newcommand{\RtpAdaptWPost}{0.095}

\newcommand{\RtpMatchedPost}{0.100}

\newcommand{\RtpPSPost}{0.095}

\newcommand{\RtpGateRecPost}{0.095}

\newcommand{\RtpNoUpdPost}{0.112}

\newcommand{\RtpAdaptLongPost}{0.093}

\newcommand{\RtpPSvsMatchedPost}{-0.0052}
\newcommand{\RtpPSvsMatchedPostSE}{0.0011}
\newcommand{\RtpPSvsMatchedPostLower}{19}
\newcommand{\RtpPSvsMatchedPostN}{24}

\newcommand{\RtpFrontierPS}{-0.0010}
\newcommand{\RtpFrontierPSSE}{0.0010}
\newcommand{\RtpFrontierPSIn}{-0.0012}
\newcommand{\RtpFrontierPSInSE}{0.0013}
\newcommand{\RtpFrontierPSInN}{19}
\newcommand{\RtpFrontierPSAllN}{24}

\newcommand{\RtpFrontierDtACI}{-0.0030}
\newcommand{\RtpFrontierDtACISE}{0.0015}

\newcommand{\RtpFrontierKernelUpd}{-0.0027}
\newcommand{\RtpFrontierKernelUpdSE}{0.0012}

\newcommand{\RtpQStaticCov}{0.863}

\newcommand{\RtpQStaticPost}{0.144}

\newcommand{\RtpQAdaptWPost}{0.092}

\newcommand{\RtpQMatchedPost}{0.097}

\newcommand{\RtpQPSPost}{0.092}

\newcommand{\RtpQKernelUpdPost}{0.090}

\newcommand{\RtpQPSvsMatchedPost}{-0.0050}
\newcommand{\RtpQPSvsMatchedPostSE}{0.0015}
\newcommand{\RtpQPSvsMatchedPostLower}{16}
\newcommand{\RtpQPSvsMatchedPostN}{24}

\newcommand{\RtpQFrontierPS}{-0.0003}
\newcommand{\RtpQFrontierPSSE}{0.0007}

\newcommand{\RtpQFrontierDtACI}{-0.0019}
\newcommand{\RtpQFrontierDtACISE}{0.0015}

\newcommand{\RtpAlphaLowPSvsMatchedPost}{-0.0039}
\newcommand{\RtpAlphaLowPSvsMatchedPostSE}{0.0009}

\newcommand{\RtpAlphaLowFrontierPS}{-0.0010}
\newcommand{\RtpAlphaLowFrontierPSSE}{0.0009}

\newcommand{\RtpAlphaHighPSvsMatchedPost}{-0.0031}
\newcommand{\RtpAlphaHighPSvsMatchedPostSE}{0.0019}

\newcommand{\RtpAlphaHighFrontierPS}{-0.0017}
\newcommand{\RtpAlphaHighFrontierPSSE}{0.0015}

\newcommand{\RtpQFracHigh}{0.037}
\newcommand{\RtpQQnineLo}{0.004}
\newcommand{\RtpQQnineHi}{0.446}
\newcommand{\RtpQCorr}{0.20}
\newcommand{\MmluN}{2565}
\newcommand{\MmluAcc}{0.605}
\newcommand{\MmluQnineLo}{0.955}
\newcommand{\MmluQnineHi}{0.980}
\newcommand{\MmluMedLo}{0.097}
\newcommand{\MmluMedHi}{0.551}

\newcommand{\TqaAcc}{0.237}
\newcommand{\TqaQnine}{1.000}
\newcommand{\RtpN}{4000}
\newcommand{\RtpQnineLo}{0.004}
\newcommand{\RtpQnineHi}{0.583}

\newcommand{\RtpFracHigh}{0.041}
\newcommand{\MmluSpreadAOne}{0.025}
\newcommand{\MmluSpreadATwo}{0.077}
\newcommand{\MmluSpreadAThree}{0.242}

\newcommand{\MmluSevenLo}{0.543}
\newcommand{\MmluSevenHi}{0.785}

\newcommand{\MmluATwoPSvsMatchedPost}{-0.0016}
\newcommand{\MmluATwoPSvsMatchedPostSE}{0.0009}

\newcommand{\MmluATwoFrontierPS}{+0.0003}
\newcommand{\MmluATwoFrontierPSSE}{0.0007}

\newcommand{\MmluAThreePSvsMatchedPost}{-0.0004}
\newcommand{\MmluAThreePSvsMatchedPostSE}{0.0013}

\newcommand{\MmluAThreeFrontierPS}{+0.0002}
\newcommand{\MmluAThreeFrontierPSSE}{0.0012}
\newcommand{\MmluAThreeFrontierPSIn}{+0.0004}
\newcommand{\MmluAThreeFrontierPSInSE}{0.0013}
\newcommand{\MmluAThreeFrontierPSInN}{23}

\newcommand{\MechSpreadMin}{0.000}
\newcommand{\MechSpreadMax}{0.858}
\newcommand{\MechGapMax}{0.0020}
\newcommand{\MechRows}{8}

\newcommand{\LarcMmluCatGap}{+0.0001}

\newcommand{\LarcMmluSubjGap}{-0.0002}

\newcommand{\LarcRtpGap}{-0.0033}
\newcommand{\LarcRtpGapSE}{0.0033}
\newcommand{\LarcRtpMinusPS}{+0.0011}
\newcommand{\LarcRtpMinusPSSE}{0.0027}
\newcommand{\LarcRtpThr}{-0.0016}
\newcommand{\LarcRtpGapIn}{-0.0109}
\newcommand{\LarcRtpGapInN}{16}

\newcommand{\LarcRtpQGap}{+0.0064}
\newcommand{\LarcRtpQGapSE}{0.0028}
\newcommand{\LarcRtpQMinusPS}{+0.0064}
\newcommand{\LarcRtpQMinusPSSE}{0.0027}
\newcommand{\LarcRtpQThr}{-0.0107}

\newcommand{\LarcTqaGap}{-0.0008}

\newcommand{\LarcSynLo}{+0.018}
\newcommand{\LarcSynHi}{+0.036}
\newcommand{\LarcSynSmallLo}{+0.17}
\newcommand{\LarcSynSmallHi}{+0.24}
\newcommand{\LarcSynSmallCovLo}{0.84}
\newcommand{\LarcSynSmallCovHi}{0.87}
\newcommand{\LarcSynCovLo}{0.896}
\newcommand{\LarcSynCovHi}{0.899}
\newcommand{\BwCells}{160}
\newcommand{\BwPass}{2}

\newcommand{\BwPSRecurMaxGap}{0.0004}
\newcommand{\BwPSRecurMaxMatched}{0.0005}
\newcommand{\BwPSSweepLo}{-0.0010}
\newcommand{\BwPSSweepHi}{+0.0030}
\newcommand{\BwPSNeffQuarter}{128}
\newcommand{\BwPSNeffOne}{385}
\newcommand{\BwKUNeffQuarterRtp}{67}
\newcommand{\BwKUNeffOneRtp}{727}
\newcommand{\BwPassGap}{-0.0074}
\newcommand{\BwPassSE}{0.0023}

\newcommand{\BwKUNeffQuarterMmlu}{30}
\newcommand{\BwKUNeffOneMmlu}{703}
\newcommand{\BwKUQuarterMmluGap}{+0.0058}
\newcommand{\BwKUQuarterMmluSE}{0.0015}
\newcommand{\BwKUQuarterQwenGap}{-0.0031}
\newcommand{\BwKUQuarterQwenSE}{0.0025}
\newcommand{\BwDtaciRecurGptTwo}{-0.0056}
\newcommand{\BwDtaciRecurGptTwoSE}{0.0003}

\newcommand{\BwDtaciRecurQwen}{-0.0061}
\newcommand{\BwDtaciRecurQwenSE}{0.0003}

\newcommand{\BwPassThr}{0.271}
\newcommand{\BwShortThr}{0.222}
\newcommand{\BwMatchedThr}{0.186}
\newcommand{\BwInsideCells}{146}
\newcommand{\BwInsidePass}{2}
\newcommand{\BwLeadGptTwoIn}{-0.0060}
\newcommand{\BwLeadGptTwoInSE}{0.0021}
\newcommand{\BwLeadGptTwoInN}{11}
\newcommand{\BwLeadGptTwoAll}{-0.0050}
\newcommand{\BwLeadGptTwoAllSE}{0.0015}

\newcommand{\BwLeadQwenIn}{-0.0051}
\newcommand{\BwLeadQwenInSE}{0.0021}
\newcommand{\BwLeadQwenInN}{17}
\newcommand{\BwLeadQwenAll}{-0.0023}
\newcommand{\BwLeadQwenAllSE}{0.0018}

\newcommand{\PoolLen}{768}
\newcommand{\RtpKernelNeff}{727}
\newcommand{\MmluKernelNeff}{703}
\newcommand{\RtpKernelRatio}{1.03}

\newcommand{\TraceRtpTop}{-0.044}

\newcommand{\TraceRtpAuc}{0.43}

\newcommand{\TraceRtpQTop}{-0.027}

\newcommand{\TraceRtpQAuc}{0.44}

\newcommand{\TraceMmluTop}{+0.024}
\newcommand{\TraceMmluTopSE}{0.0035}
\newcommand{\TraceMmluAuc}{0.51}

\newcommand{\TraceSynTopLo}{+0.047}
\newcommand{\TraceSynTopHi}{+0.071}

\newcommand{\TraceRealRhoLo}{-0.66}
\newcommand{\TraceRealRhoHi}{-0.25}
\newcommand{\TraceRtpErrFirst}{0.121}
\newcommand{\TraceRtpErrLater}{0.090}
\newcommand{\TraceRtpDBefore}{1.354}
\newcommand{\TraceRtpDAfter}{1.362}
\newcommand{\TraceMmluErrFirst}{0.095}
\newcommand{\TraceMmluErrLater}{0.098}

\theoremstyle{acmplain}
\newtheorem{assumption}{Assumption}
\theoremstyle{acmdefinition}
\newtheorem{remark}{Remark}
\begin{document}

\title[PromptShift-CRC]{PromptShift-CRC: What Drift-Aware Weighting Adds to Online Conformal Risk Control for Foundation Models}

\author{Jeffery Opoku}
\authornote{Corresponding Author.}
\orcid{0009-0002-4508-0027}
\affiliation{%
  \institution{School of Mathematical and Statistical Sciences, The University of Texas Rio Grande Valley}
  \city{Edinburg}
  \state{Texas}
  \country{USA, jeffery.opoku01@utrgv.edu}
}

\author{David Banahene}
\affiliation{%
  \institution{Department of Biostatistics, Robert Stempel College of Public Health and Social Work, Florida International University}
  \city{Miami}
  \state{Florida}
  \country{USA, dbana006@fiu.edu}
}

\renewcommand{\shortauthors}{Opoku and Banahene}

\begin{abstract}
A risk threshold calibrated on old prompts can stop being valid when prompts to a foundation model drift. We study PromptShift-CRC, a rule that weights past prompts by feature similarity, falls back to recent prompts when the current prompt is unlike the past, and updates its level online. We prove that the update controls the long-run violation rate for any weights, so marginal coverage cannot separate weighting schemes. We give a conditional bound at a fixed level whose cost is a weighted total variation distance, bounded by a computable prompt-drift term up to an unknown Lipschitz constant. We propose a comparison protocol: a matched short-window baseline, a validity and sharpness frontier, 30-seed synthetic streams, and real-output streams ordered without using the model's output. On MMLU, TruthfulQA (Qwen2.5-1.5B) and RealToxicityPrompts (GPT-2, replicated with a second generator and classifier), every online method reaches the target. PromptShift-CRC stays within 0.0055 of the Adaptive CRC frontier and avoids the worst window, but DtACI matches or beats it in most settings. The rule performs like a mid-ranked window without having to choose one. Tests fixed in advance to favor the weights show no consistent gain over the frontier. Code, outputs and streams are released.
\end{abstract}

\maketitle


\section{Introduction}\label{sec:intro}

A foundation model in use sees a stream of prompts that changes over time. New users arrive, topics change, policies change, and the model meets kinds of request that were rare when its risk was first measured. Conformal prediction and conformal risk control give a way to set a threshold on a risk score so that the long-run rate of violations stays near a chosen level $\alpha$, with no assumption on the model \citep{vovk2005algorithmic,angelopoulos2021gentle,angelopoulos2024conformal}. The guarantee is clearest when the calibration examples and the future examples are exchangeable, and a prompt stream with drift is the case where they are not.

Two families of tools address this. Weighted conformal methods change the calibration distribution with importance weights \citep{tibshirani2019conformal,barber2023beyond}. Online methods adjust the target level after each observed violation \citep{gibbs2021adaptive,gibbs2024online,angelopoulos2023pid}. A natural proposal for foundation models is to combine them: weight past prompts by how similar they are to the current prompt, fall back to recent prompts when the current prompt looks unlike the past, and update the level online. We call this combination PromptShift-CRC. This paper asks what each part of such a method contributes, and it answers with theory, a fixed evaluation protocol, and experiments on real model outputs.

\paragraph{What we find.}
The answer is more modest than we expected.
\begin{enumerate}[label=(\roman*)]
  \item With the online update used here, long-run risk control is a property of the update and not of the weights. For every weighting scheme the average violation rate is within $(\max\{\alpha,1-\alpha\}+\gamma)/(T\gamma)$ of $\alpha$ (Theorem~\ref{thm:longrun}), so marginal coverage cannot tell weighting schemes apart. In all of our experiments, every method that uses this update is within the bound, as Theorem~\ref{thm:longrun} requires; DtACI and quantile tracking, which use other updates, are within it as well.
  \item The weights do matter for \emph{local} behaviour and for sharpness. We give a bound at a fixed level in which the cost of using calibration points from a different regime is a weighted sum of total variation distances (Theorem~\ref{thm:cond}) and, under a smoothness condition, a computable \emph{prompt drift} $D_t$ built from the weights actually used (Corollary~\ref{cor:drift}). This is a statement about how a penalty scales and not a certificate.
  \item A fair comparison needs a baseline that uses a short calibration window. We call Adaptive conformal inference with a window equal to PromptShift-CRC's recent window the \emph{matched} baseline, and we also trace the validity and sharpness trade-off of Adaptive CRC by sweeping its window. In our experiments PromptShift-CRC and its variants stay within a small margin of Adaptive CRC: on average over streams, the post-shift violation of PromptShift-CRC differs from the Adaptive CRC curve at equal sharpness by at most \PSfrontierMaxAbs{} in every setting. DtACI does as well or better in most. The one advantage of PromptShift-CRC is that it performs like a mid-ranked window without a window having to be chosen (Section~\ref{sec:res-window}). Tests fixed in advance to give the weights a better chance do not change this (Sections~\ref{sec:res-real} to~\ref{sec:res-bw}).
  \item These results hold on synthetic streams with 30 seeds each and on streams built from real model outputs (MMLU and TruthfulQA with a 1.5-billion-parameter model, and RealToxicityPrompts with GPT-2 continuations scored by a public toxicity classifier, replicated with a second generator and classifier), ordered by variables that do not depend on the model's output.
\end{enumerate}

\paragraph{Contributions.}
\begin{enumerate}[label=(\arabic*)]
  \item A statement of the PromptShift-CRC rule, in which the weights, the drift gate and the recency fallback depend only on prompt features, with the finite-sample correction needed for a valid conditional bound (Section~\ref{sec:method}).
  \item A conditional guarantee at a fixed level, a drift penalty computed from the final weights, a long-run guarantee for the online rule that holds for any weights, and a bound on its conditional violation probability at the running level, with proofs (Section~\ref{sec:theory}).
  \item An evaluation protocol with a matched-window baseline and a window sweep, a validity and sharpness frontier, multi-seed synthetic streams, orderings that do not use the outcome, real model outputs and, for the real streams, settings fixed before the results were seen (Section~\ref{sec:protocol}).
  \item An empirical answer, including where the drift-aware components help a little, where they do not, and what a practitioner should use (Section~\ref{sec:results}).
\end{enumerate}

\section{Related work}\label{sec:related}

\paragraph{Conformal prediction and conformal risk control.}
Split conformal prediction gives finite-sample coverage under exchangeability \citep{vovk2005algorithmic,angelopoulos2021gentle}. Conformal risk control extends it from coverage to any bounded monotone loss \citep{angelopoulos2024conformal}. Under covariate shift, weighting the calibration scores by the likelihood ratio restores coverage when the ratio is known \citep{tibshirani2019conformal}; \citet{barber2023beyond} show that fixed weights lose coverage by an amount governed by the total variation distance between the data distribution and a swapped version of it, which is the form of guarantee we use. \citet{angelopoulos2024conformal} also give a weighted version of conformal risk control under covariate shift, and \citet{gibbs2025conditional} obtain coverage conditional on features over a class of covariate shifts. \citet{cauchois2024robust} and \citet{podkopaev2022tracking} treat other kinds of shift. In JAIR itself, conformal prediction has appeared for regression with nearest-neighbour nonconformity measures under exchangeability \citep{papadopoulos2011regression}.

\paragraph{Online conformal methods.}
Adaptive conformal inference updates the miscoverage level after each step and has a deterministic long-run bound for any data sequence \citep{gibbs2021adaptive}. Later work chooses the step size automatically or makes it decay, aggregates several step sizes, tracks the quantile directly, or adds integral and forecasting terms \citep{zaffran2022adaptive,gibbs2024online,bhatnagar2023improved,angelopoulos2024decaying,angelopoulos2023pid}. For risk control with losses beyond coverage, \citet{feldman2023online} give an online procedure, \citet{zecchin2024localized} make the control local by letting the threshold depend on features, and \citet{timans2025continuous} monitor risk violations under unknown shift. A recency window is the usual device against non-stationarity in other online-learning settings, for example sliding-window Thompson sampling for non-stationary bandits \citep{trovo2020sliding}; our matched short-window baseline plays the same role for conformal risk control. Our long-run result is the Gibbs and Cand\`es bound applied to our update, and we do not claim it as new.

\paragraph{Weighting and localisation under shift.}
Similarity weights, localised calibration and density-ratio weights are standard ways to adapt a calibration set to the current input \citep{guan2023localized,bastani2022practical,zecchin2025weighted}. For language models, \citet{mohri2024language} give conformal factuality guarantees, \citet{lin2026domain} weight calibration scores for domain shift, and \citet{hu2026cofact} estimate density ratios online to control hallucination under covariate shift. PromptShift-CRC is closest to these: it is a feature-based weighting rule with a recency fallback and an online level update. We do not claim a new weighting idea. Our contribution is an account of what such a rule can and cannot guarantee, and a protocol that compares it with strong, simple baselines.

\paragraph{Language-model uncertainty and benchmarks.}
Conformal ideas have been applied to multiple-choice question answering, generation and open-ended outputs \citep{kumar2023conformal,noorani2025missingmass}. We use MMLU and TruthfulQA \citep{lin2021truthfulqa}. Earlier versions of this work also used streams built from RealToxicityPrompts, SummEval and LongHalQA \citep{gehman2020realtoxicityprompts,fabbri2020summeval,qiu2024longhalqa}. Those streams contain no model outputs and are reported in the appendix with that label. Evidence that prompts and model behaviour shift in practice comes from studies of code tasks, semantic rewording and non-deterministic drift \citep{li2024codeuncertainty,li2025pbss,nicholson2026nondeterministic}.

\section{Setting and method}\label{sec:method}

At step $t$ the system receives a prompt with feature vector $E_t\in\mathbb R^d$ (a sentence embedding of the prompt), a fixed model produces a response, and a risk score $R_t\in[0,1]$ is revealed after the decision. Large values mean more risk, for example $R_t=1-p_t$ where $p_t$ is the probability the model gave to the correct answer, or the 0/1 error. The aim is a threshold $q_t$, chosen from $E_{1:t}$ and $R_{1:t-1}$, so that the violation indicator $\mathrm{err}_t=\mathbf 1\{R_t>q_t\}$ has long-run mean close to $\alpha$ and, more usefully, has no long stretch with a much higher rate. With the indicator loss this is the conformal-prediction case of conformal risk control \citep{angelopoulos2024conformal} and our guarantees and experiments cover it. A threshold above a cut-off $\tau$ triggers an action such as abstaining, retrieving or asking for review. We write the fraction of such steps as the \emph{intervention rate}.

\paragraph{Weights.}
Given a pool of past steps $s\in\mathcal C_t=\{t-3W,\dots,t-1\}$ (with $W$ a window length), the method forms kernel weights from the distance of each past prompt to the current prompt,
\[
 \tilde s_{t,s}=\exp\{-\|E_s-E_t\|^2/h^2\}+10^{-8},\qquad s_{t,s}=\tilde s_{t,s}\Big/\textstyle\sum_{j\in\mathcal C_t}\tilde s_{t,j},
\]
with bandwidth $h>0$ and a floor of $10^{-8}$ for numerical safety (the code adds it to the kernel weights on the real streams and to the mixed weights on the synthetic streams). The \emph{pre-gate drift} is the weighted mean distance
$D^{\mathrm{pre}}_t=\sum_s s_{t,s}\|E_s-E_t\|$. A logistic gate $g_t=\{1+\exp[\lambda(D^{\mathrm{pre}}_t-c)]\}^{-1}\in(0,1)$ mixes the kernel weights with uniform weights on the most recent $R$ steps $\mathcal R_t\subseteq\mathcal C_t$:
\[
 p_{t,s}\ \propto\ g_t\,s_{t,s}+(1-g_t)\,\frac{\mathbf 1\{s\in\mathcal R_t\}}{|\mathcal R_t|}+10^{-8},\qquad \textstyle\sum_s p_{t,s}=1.
\]
When the current prompt is close to many past prompts ($D^{\mathrm{pre}}_t$ small) the gate is near one and the rule borrows from all similar prompts. When it is far from the past, the rule leans on the most recent steps. The weights depend on prompt features only and never on scores. We report the \emph{post-gate drift} $D_t=\sum_s p_{t,s}\|E_s-E_t\|$ and the effective size $n_{\mathrm{eff},t}=1/\sum_s p_{t,s}^2$.

\paragraph{Threshold and update.}
The threshold is the quantile of the weighted past scores at level $\ell_t=(1-\alpha_t)(1+m_t)$, with $m_t=\max_s p_{t,s}$ and the conventions of Section~\ref{sec:theory}, and $\alpha_{t+1}=\alpha_t+\gamma(\alpha-\mathrm{err}_t)$ with $\alpha_1=\alpha$. The update has no projection. The factor $1+m_t$ is the Barber et al.\ correction and equals the usual $(1+1/n)$ for uniform weights.

\paragraph{Baselines and variants.}
\emph{Static}: the first calibration segment only. \emph{Rolling}: uniform weights on the last $W$ steps, fixed level. \emph{Adaptive CRC}: rolling window plus the online update, run with the full window $W$ and with a window equal to the recent window $R$ (\emph{matched}), and with a sweep of other windows. \emph{Representation-weighted}: kernel weights over a pool of $3W$ steps, with ($\texttt{rep\_aci}$) and without the update. \emph{PromptShift-CRC}: the full rule above. \emph{Ablations}: no kernel weights (uniform weights inside the gate), no gate (fixed mix $g_t=0.5$), no recency fallback ($g_t\equiv1$), and no online update. Three published online methods serve as stronger baselines: quantile tracking (the P term of conformal PID control, \citeauthor{angelopoulos2023pid}~\citeyear{angelopoulos2023pid}), DtACI \citep{gibbs2024online} with the constants of its arXiv version ($I=500$ step sizes $0.001\cdot 2^i$) and the recent window $R$ as its calibration window, and a density-ratio weighted rule in the spirit of \citet{lin2026domain,hu2026cofact} (simplified: a logistic regression separating the recent window from the older pool, refit every 25 steps). Localised adaptive risk control \citep{zecchin2024localized} is compared in Section~\ref{sec:res-larc}, with the authors' update and defaults.


\section{Theory}\label{sec:theory}

We use the notation of Section~\ref{sec:method}: weights $p_{t,s}$ that depend on the features $E_{1:t}$ only, $m_t=\max_{s<t}p_{t,s}$ and $n_{\mathrm{eff},t}=1/\sum_s p_{t,s}^2$.

\paragraph{The threshold.}
For a miscoverage level $\alpha_t$ the threshold is the weighted quantile of the past scores at level
\[
 \ell_t=(1-\alpha_t)(1+m_t),\qquad
 q_t=\inf\Bigl\{x:\ \textstyle\sum_{s<t}p_{t,s}\mathbf 1\{R_s\le x\}\ge \ell_t\Bigr\},
\]
with the conventions $q_t=+\infty$ if $\ell_t>1$ and $q_t=-\infty$ if $\ell_t\le 0$.
The factor $1+m_t$ is the finite-sample correction of Barber et al.\ (2023): it equals the
quantile, at level $1-\alpha_t$, of $\sum_{s<t}\tilde w_{t,s}\delta_{R_s}+\tilde w_{t,t}\delta_{+\infty}$ with
$\tilde w_{t,s}=p_{t,s}/(1+m_t)$ and $\tilde w_{t,t}=m_t/(1+m_t)$. For uniform weights on $n$ scores
this is the usual split-conformal level $(1-\alpha)(1+1/n)$.
The level follows the online rule $\alpha_{t+1}=\alpha_t+\gamma(\alpha-\mathrm{err}_t)$,
$\mathrm{err}_t=\mathbf 1\{R_t>q_t\}$, $\alpha_1=\alpha$, \emph{with no projection}.

\subsection{A conditional guarantee at a fixed level}

\begin{assumption}[Conditional independence]\label{ass:ci}
Given the features $E_{1:t}$, the scores $R_1,\dots,R_t$ are independent and
$R_s\mid E_{1:t}\sim\kappa(E_s,\cdot)$ for a fixed Markov kernel $\kappa$ (the law of the risk of a prompt with feature $e$).
Write $L_s=\kappa(E_s,\cdot)$.
\end{assumption}

Assumption~\ref{ass:ci} holds, for example, when prompts are independent draws and the model's
randomness is independent across prompts. It does not require the laws $L_s$ to be equal.

\begin{theorem}[Conditional coverage at a fixed level]\label{thm:cond}
Fix $t$ and $\alpha\in(0,1)$ and let $q_t^{+}$ be the threshold above with $\alpha_t\equiv\alpha$.
Under Assumption~\ref{ass:ci},
\[
 \mathbb P\bigl(R_t>q_t^{+}\ \big|\ E_{1:t}\bigr)\ \le\ \alpha+\frac{2}{1+m_t}\sum_{s<t}p_{t,s}\,
 d_{\mathrm{TV}}(L_s,L_t)\ \le\ \alpha+2\sum_{s<t}p_{t,s}\,d_{\mathrm{TV}}(L_s,L_t).
\]
In particular, if $L_s=L_t$ for every $s$ with $p_{t,s}>0$, then $\mathbb P(R_t>q_t^{+}\mid E_{1:t})\le\alpha$.
\end{theorem}

\begin{proof}
Condition on $E_{1:t}=e$; the weights $p_{t,s}$ are then fixed numbers. Apply Theorem~2 of
Barber, Cand\`es, Ramdas and Tibshirani (2023) and the remark following it (their Section~4.1) that the same
bound holds for nonexchangeable split conformal prediction, with the scores $R_s$ in place of the residuals and
no fitted function. The sequence is $(R_1,\dots,R_t)$ with conditional joint law $\mathcal L(R_{1:t}\mid e)$;
the fixed weights are $\hat w_s=p_{t,s}/m_t\in[0,1]$ for $s<t$ (their Section~3.1) and the test weight is $1$. Their normalised weights (their equation~(10)) are
$\hat w_s/(\sum_{j<t}\hat w_j+1)=p_{t,s}/(1+m_t)=\tilde w_{t,s}$, and their threshold (their equation~(11)) is the quantile
$q_t^{+}$ defined above. The theorem gives
$\mathbb P(R_t\le q_t^{+}\mid e)\ge 1-\alpha-\sum_{s<t}\tilde w_{t,s}\,d_{\mathrm{TV}}(\mathbf R,\mathbf R^{s})$,
where $\mathbf R=(R_1,\dots,R_t)$ and $\mathbf R^{s}$ is $\mathbf R$ with entries $s$ and $t$ swapped.
Under Assumption~\ref{ass:ci}, $\mathbf R$ has product law $\bigotimes_j L_j$ and $\mathbf R^{s}$ has the product law with
$L_s$ and $L_t$ exchanged, and these differ only in coordinates $s$ and $t$, so
$d_{\mathrm{TV}}(\mathbf R,\mathbf R^{s})\le d_{\mathrm{TV}}(L_s\otimes L_t,\ L_t\otimes L_s)\le 2\,d_{\mathrm{TV}}(L_s,L_t)$
(a product of two laws changes by at most the sum of the coordinate changes, by the triangle inequality).
Substituting gives the first bound; the second uses $1/(1+m_t)\le 1$.
\end{proof}

Theorem~\ref{thm:cond} holds at a fixed level and does not apply to the data-dependent level $\alpha_t$ of the online rule;
Theorem~\ref{thm:longrun} below covers that rule for every choice of weights, including weights that depend on past scores, and, for weights that depend on the features only, Proposition~\ref{prop:running} bounds its conditional violation probability at the running level.

\begin{lemma}[Size of the correction]\label{lem:m}
$1/n_{\mathrm{eff},t}\le m_t\le n_{\mathrm{eff},t}^{-1/2}$.
\end{lemma}
\begin{proof}
$\sum_s p_{t,s}^2\le m_t\sum_s p_{t,s}=m_t$ gives the first inequality, and $m_t^2\le\sum_s p_{t,s}^2$ gives the second.
\end{proof}
The correction factor in $\ell_t$ therefore lies between $1+1/n_{\mathrm{eff},t}$ and $1+n_{\mathrm{eff},t}^{-1/2}$: concentrated weights are
charged in the level, not only in the variance of the weighted quantile.

\begin{corollary}[Drift controls the penalty]\label{cor:drift}
Suppose, in addition, $d_{\mathrm{TV}}\bigl(\mathcal L(R\mid E=e),\mathcal L(R\mid E=e')\bigr)\le \Lambda\|e-e'\|$ for all
$e,e'$ in the support of the features. Then, for the final (post-gate) weights,
\[
 \mathbb P(R_t>q_t^{+}\mid E_{1:t})\le\alpha+\frac{2\Lambda}{1+m_t}\,D_t,\qquad
 D_t=\sum_{s<t}p_{t,s}\,\|E_s-E_t\|.
\]
\end{corollary}

The quantity $D_t$ is the \emph{prompt drift} of the paper, computed with the weights that are actually used
to form the threshold (after the gate and recency mixing), so the diagnostic is the one that appears in the bound.
The Lipschitz constant $\Lambda$ is a property of the data-generating process; it is not estimated by the method,
and the bound is a statement about how the penalty scales, not a computable certificate.

Corollary~\ref{cor:drift} and Lemma~\ref{lem:m} show the trade-off that any drift-aware rule faces: concentrating weight on prompts close to
$E_t$ lowers $D_t$ and raises $m_t$ (lowers $n_{\mathrm{eff},t}$). Appendix~\ref{app:ess} gives the sampling error of the weighted quantile in the
no-drift case, which grows as $n_{\mathrm{eff},t}^{-1/2}$. Neither bound orders the rules, so the comparison between rules is empirical
(Section~\ref{sec:results}).

\subsection{A long-run guarantee for the online update}

\begin{theorem}[Long-run control, any weights]\label{thm:longrun}
For every $T\ge1$, every score sequence $R_1,\dots,R_T\in[0,1]$ and every choice of weights,
\[
 \Bigl|\frac1T\sum_{t=1}^T\mathrm{err}_t-\alpha\Bigr|\ \le\ \frac{\max\{\alpha_1,1-\alpha_1\}+\gamma}{T\gamma}.
\]
\end{theorem}

\begin{proof}
We first show $\alpha_t\in[-\gamma,1+\gamma]$ for all $t$, by induction from $\alpha_1=\alpha\in(0,1)$.
Suppose $\alpha_t\in[-\gamma,1+\gamma]$.
If $\alpha_t<0$ then $\ell_t=(1-\alpha_t)(1+m_t)>1$, so $q_t=+\infty$, $\mathrm{err}_t=0$ and
$\alpha_{t+1}=\alpha_t+\gamma\alpha\in(\alpha_t,\gamma\alpha)\subset[-\gamma,1+\gamma]$.
If $\alpha_t>1$ then $\ell_t<0$, so $q_t=-\infty$, $\mathrm{err}_t=1$ and
$\alpha_{t+1}=\alpha_t-\gamma(1-\alpha)\in[1-\gamma(1-\alpha),\alpha_t)\subset[-\gamma,1+\gamma]$.
If $\alpha_t\in[0,1]$ then $\alpha_{t+1}\in[-\gamma(1-\alpha),\,1+\gamma\alpha]\subset[-\gamma,1+\gamma]$.
Summing the update, $\alpha_{T+1}-\alpha_1=\gamma\sum_{t\le T}(\alpha-\mathrm{err}_t)$, so
$\frac1T\sum_t\mathrm{err}_t-\alpha=(\alpha_1-\alpha_{T+1})/(T\gamma)$, and
$\alpha_1-\alpha_{T+1}\le\alpha_1+\gamma$ and $\alpha_{T+1}-\alpha_1\le 1+\gamma-\alpha_1$, so
$|\alpha_1-\alpha_{T+1}|\le\max\{\alpha_1,1-\alpha_1\}+\gamma$, which gives the bound.
\end{proof}

This is Proposition~4.1 of Gibbs and Cand\`es (2021) specialised to our update; we give the proof because the
two conventions at the extremes ($\ell_t>1$ and $\ell_t<0$) are what the proof uses, and the correction factor
$1+m_t$ leaves them unchanged. The bound is the same for every weighting scheme: \emph{long-run marginal
coverage is a property of the update, not of PromptShift-CRC's weights}. The bound is $0.018$ to $0.024$ for the synthetic streams, $0.020$ for MMLU, $0.012$ for RealToxicityPrompts and $0.043$ for TruthfulQA.

\begin{proposition}[Conditional bound at the running level]\label{prop:running}
Suppose that every law $L_s$ with $s<t$ and $p_{t,s}>0$ is continuous. Fix $t$, $\eta\in(0,1)$ and an integer $M\ge 1$, and write
\[
 \varepsilon_t=\sqrt{\frac{\log(2M/\eta)}{2\,n_{\mathrm{eff},t}}}+\frac1M .
\]
Under Assumption~\ref{ass:ci}, conditionally on $E_{1:t}$, with probability at least $1-\eta$ over $R_{1:t-1}$,
\[
 \mathbb P\bigl(R_t>q_t\ \big|\ E_{1:t},R_{1:t-1}\bigr)\ \le\ \max\{\alpha_t,0\}+\sum_{s<t}p_{t,s}\,d_{\mathrm{TV}}(L_s,L_t)+\varepsilon_t ,
\]
where $q_t$ is the threshold of the online rule at its own level $\alpha_t$.
\end{proposition}

\begin{proof}
Condition on $E_{1:t}$, so the weights $p_{t,s}$ are fixed numbers. The level $\alpha_t$ and the threshold $q_t$ are functions of $R_{1:t-1}$ (and of the features), and under Assumption~\ref{ass:ci} the score $R_t$ is independent of $R_{1:t-1}$ with law $L_t$. Hence $\mathbb P(R_t>q_t\mid E_{1:t},R_{1:t-1})=1-F_t(q_t)$, where $F_s$ is the distribution function of $L_s$.

Let $\bar F=\sum_{s<t}p_{t,s}F_s$, which is continuous because every $F_s$ in the sum is, and $\hat F(x)=\sum_{s<t}p_{t,s}\mathbf 1\{R_s\le x\}$. For a fixed $x$, $\hat F(x)-\bar F(x)$ is a sum of independent, mean-zero terms $p_{t,s}(\mathbf 1\{R_s\le x\}-F_s(x))$, each in an interval of length $p_{t,s}$, so Hoeffding's inequality gives $\mathbb P(|\hat F(x)-\bar F(x)|>\varepsilon)\le 2\exp(-2\varepsilon^2 n_{\mathrm{eff},t})$. The grid argument in the proof of Proposition~\ref{prop:ess} uses only independence of the summands, continuity of the target function (so that $\bar F(x_j)=j/M$) and the absence of atoms in each $L_s$ (so that almost surely no $R_s$ equals any $x_j$). The grid $x_j=\bar F^{-1}(j/M)$ depends on $E_{1:t}$ only, so it gives an event $\mathcal E$ with $\mathbb P(\mathcal E\mid E_{1:t})\ge 1-\eta$ on which $\sup_x|\hat F(x)-\bar F(x)|\le\varepsilon_t$. The event $\mathcal E$ does not involve $\alpha_t$ or $q_t$, and the supremum is over all $x$, so the bound holds at the random point $q_t$.

On $\mathcal E$ there are three cases. If $\ell_t>1$ then $q_t=+\infty$, $1-F_t(q_t)=0$, and the right-hand side is nonnegative; this case contains every $\alpha_t<0$, which the unprojected update can produce (Theorem~\ref{thm:longrun}), and without the positive part the bound could fail there. If $\ell_t\le 0$ then $\alpha_t\ge 1$ and the claim is trivial because $1-F_t(q_t)\le 1\le\alpha_t$. If $0<\ell_t\le 1$ then $q_t$ is finite, $\hat F(q_t)\ge\ell_t$ by the definition of the quantile and right-continuity of $\hat F$, and $\alpha_t\in[0,1)$ (since $\alpha_t<0$ would force $\ell_t>1$) gives $\ell_t=(1-\alpha_t)(1+m_t)\ge 1-\alpha_t=1-\max\{\alpha_t,0\}$. Then
\[
 F_t(q_t)=\bar F(q_t)+\sum_{s<t}p_{t,s}\bigl(F_t(q_t)-F_s(q_t)\bigr)\ \ge\ \hat F(q_t)-\varepsilon_t-\sum_{s<t}p_{t,s}\,d_{\mathrm{TV}}(L_s,L_t)\ \ge\ 1-\alpha_t-\varepsilon_t-\sum_{s<t}p_{t,s}\,d_{\mathrm{TV}}(L_s,L_t),
\]
because $|F_t(x)-F_s(x)|\le d_{\mathrm{TV}}(L_s,L_t)$ for every $x$. Rearranging gives the claim.
\end{proof}

Taking expectations over $R_{1:t-1}$, $\mathbb P(R_t>q_t\mid E_{1:t})\le\mathbb E[\max\{\alpha_t,0\}\mid E_{1:t}]+\sum_{s<t}p_{t,s}d_{\mathrm{TV}}(L_s,L_t)+\varepsilon_t+\eta$. Against Theorem~\ref{thm:cond} the drift term enters with coefficient $1$ instead of $2/(1+m_t)$, and $\varepsilon_t+\eta$ is the price of the data-dependent level. The proof uses only that $\alpha_t$ is a function of the past, so the bound holds for any history-dependent level and any feature-measurable weights. Keeping the factor $1+m_t$ would improve the finite-threshold case by $m_t(1-\alpha_t)\le n_{\mathrm{eff},t}^{-1/2}$ (Lemma~\ref{lem:m}), which we omit. The bound does not apply to the $0/1$ error risk, whose distribution has atoms.

\begin{remark}[What the bound does and does not say]\label{rem:running}
The bound is relative to the running level $\alpha_t$, not to $\alpha$, and it is a structural statement and not a certificate: for $n_{\mathrm{eff},t}=375$, $M=100$ and $\eta=0.05$ the sampling term is $\varepsilon_t\approx 0.115$ ($n_{\mathrm{eff},t}=375$ is the mean effective size of PromptShift-CRC on MMLU). The drift term can fall (under the Lipschitz condition of Corollary~\ref{cor:drift}) and the sampling term rises (Lemma~\ref{lem:m}) when weight is concentrated on prompts close to $E_t$. The bound does not say that $\alpha_t$ stays close to $\alpha$. For that we have Theorem~\ref{thm:longrun} and the following identity, which follows by summing the update over the window, as in the proof of Theorem~\ref{thm:longrun}: for any window $a\le t\le b$ of length $K=b-a+1$,
\[
 \frac1K\sum_{t=a}^{b}\mathrm{err}_t-\alpha=\frac{\alpha_a-\alpha_{b+1}}{K\gamma}.
\]
The local violation rate of the online rule over a window is therefore $\alpha$ plus the drop $\alpha_a-\alpha_{b+1}$ of the level over that window, divided by $K\gamma$. For example, at $\alpha=0.1$ and $\gamma=0.02$ a maximum local violation of $0.20$ over $K=120$ steps is a drop of $0.24$.
\end{remark}

\begin{remark}[What gating can and cannot do]\label{rem:gate}
Theorem~\ref{thm:longrun} gives every weighting scheme the same long-run bound, so long-run marginal coverage cannot distinguish schemes.
By Theorem~\ref{thm:cond} and Lemma~\ref{lem:m}, a gate that moves weight toward recent or similar prompts can reduce the drift penalty
$D_t$ while it reduces $n_{\mathrm{eff},t}$. We therefore do not claim that the gate or the representation weights improve
risk control; whether they improve local validity or sharpness is measured, not assumed.
\end{remark}


\section{Evaluation protocol}\label{sec:protocol}

An earlier version of this work compared PromptShift-CRC with three baselines on single runs of hand-built streams. On re-examination we found that a short-window baseline matches it, that single runs hide the seed-to-seed spread, and that the benchmark streams used no model outputs. This section describes what we do instead. Every real-stream setting was fixed before its results were computed; the synthetic constants are inherited from the earlier version and were chosen while those generators were built (Section~\ref{sec:limits}).

\paragraph{Metrics.}
Marginal coverage ($1-$ mean of $\mathrm{err}_t$) is reported, but Theorem~\ref{thm:longrun} bounds its distance from $1-\alpha$ for every online method, so it cannot rank them. The metrics that can are the following.
\begin{itemize}[nosep]
  \item \emph{Post-shift violation}: the mean of $\mathrm{err}_t$ over the first $H$ steps after each known regime boundary ($H=150$ for MMLU, RealToxicityPrompts and the synthetic streams, $50$ for TruthfulQA). Lower is better and the target is $\alpha=0.1$.
  \item \emph{Maximum local violation}: the largest mean of $\mathrm{err}_t$ over any window of $120$ steps ($60$ for TruthfulQA).
  \item \emph{Mean threshold}: the average of $\min\{q_t,1\}$, a measure of sharpness; a lower threshold means fewer unnecessary interventions.
  \item \emph{Intervention rate}: the fraction of steps with $q_t>\tau$, with $\tau=0.5$ fixed in advance.
  \item \emph{Frontier gap}: for each stream, Adaptive CRC is run with five windows. This traces a curve of post-shift violation against mean threshold. A method's frontier gap is its post-shift violation minus the value of that curve, linearly interpolated at the method's own mean threshold. A negative gap means the method does better than Adaptive CRC at the same sharpness. When a method's mean threshold lies outside the range spanned by the five windows, the curve is clipped to its nearest end rather than extrapolated, which can favour or disfavour the method depending on the slope of the curve. For the real streams we therefore also report the gap over the streams whose threshold lies inside the range.
\end{itemize}

\paragraph{Streams.}
\emph{Synthetic}: five generators, each with three era boundaries within which the regime is fixed or alternates between two regimes every $50$ to $80$ steps, run with $30$ seeds ($0$ to $29$). These are the generators of the earlier version, with a deterministic per-task seed (the earlier script derived the seed from Python's per-process string hash, so its table could not be reproduced).
\emph{Real model outputs}: we run Qwen2.5-1.5B-Instruct \citep{qwen2025technical} on CPU with one fixed prompt template for each task and record the probability the model gives to each answer option. For MMLU \citep{hendrycks2021measuring} we use 45 randomly chosen questions per subject (a fixed seed), $2565$ in all, and for TruthfulQA \citep{lin2021truthfulqa} all $817$ questions in the multiple-choice setting, scoring each answer choice by the summed log-probability of its text. The risk is $1-p(\text{correct})\in[0,1]$, and as a second risk we use the 0/1 error. Features are unit-norm all-MiniLM-L6-v2 sentence embeddings of the question text \citep{reimers2019sentence}.
For RealToxicityPrompts \citep{gehman2020realtoxicityprompts} we draw $4000$ of the $99{,}442$ prompts at random (a fixed seed), continue each with GPT-2 \citep{radford2019language} (one nucleus sample with $p=0.9$ and $20$ new tokens), and score the continuation alone with the \texttt{unitary/toxic-bert} toxicity output \citep{hanu2020detoxify}; the features are embeddings of the prompt text, and the risk is that toxicity probability. A replication uses Qwen2.5-0.5B continuations scored by \texttt{unitary/unbiased-toxic-roberta} on the same prompts. The \emph{orderings} use only the benchmark's own metadata or the question text, never the model's output: random orders of the four broad MMLU categories ($24$ permutations), the $24$ permutations of the four quartile bands of the prompt toxicity annotation (RealToxicityPrompts; the annotation describes the prompt, not the model's output), ascending and descending prompt toxicity and of the $57$ subjects ($20$ orders), random orders of the TruthfulQA categories ($20$), embedding-direction orders (first principal component, forward and reversed), question-length order, and $20$ random permutations as a control with no shift.

\paragraph{Pre-declared settings.}
All methods use $\alpha=0.1$, the unprojected update with $\gamma$ fixed per experiment ($0.025$ for the synthetic pilot; $0.02$ for the other synthetic streams, MMLU and RealToxicityPrompts; $0.03$ for TruthfulQA, which has $n=817$), the correction factor $1+m_t$, and a pool of $3W$ steps for every weighted rule. On the real streams the calibration prefix and window are $256$ with a recent window of $128$ (MMLU) and $100$ with $50$ (TruthfulQA). The kernel bandwidth is the median pairwise distance $h$ among the first prefix prompts of that stream, the gate centre is $1.1h$ and the gate slope is $2.5/h$. These rules were written down before any real-stream result was computed. The synthetic-stream constants are those of the earlier scripts and were not changed; the earlier version chose them while developing the generators (Section~\ref{sec:limits}). The analyses added after the first review (MMLU at further levels, the kernel bandwidth sweep, the recurring-regime orderings, the test of the drift and the comparison with L-ARC) were written down with their decision rules before they were run, in a time-stamped file in the repository linked on acceptance. Where the outcome did not fit a sentence we had written in advance, we say so in the text. Two of the five rules turned out to be under-specified (no requirement of a consistent sign across families, and no correction for the number of cells); in both cases we report the rule's own outcome and our reading of it separately.

\paragraph{Uncertainty.}
For synthetic streams we report the mean over $30$ seeds with the standard deviation and compare methods by paired differences within seed (mean difference, standard error, and the number of seeds in which one method is lower). For the real streams, the unit of replication is the ordering. For a single ordered stream, standard errors of coverage use a moving-block bootstrap of the violation indicator with block length equal to the recent window.


\section{Results}\label{sec:results}

\subsection{Synthetic streams}\label{sec:res-syn}

Each of the five synthetic generators has three era boundaries at which the prompt features and the risk change; within an era the regime is fixed or alternates between two regimes every $50$ to $80$ steps. Table~\ref{tab:synthetic} reports the mean over \nSeeds{} seeds, and Figure~\ref{fig:frontier-syn} plots post-shift violation against sharpness.

\paragraph{The update does the long-run work.}
Static calibration reaches coverage \synStaticCov{} and rolling calibration \synRollingCov{}, below the target of $0.90$. Every method with the online update is within $0.02$ of the target, consistent with Theorem~\ref{thm:longrun} (the bound is $0.018$ to $0.024$ for these streams and is listed for every experiment in Table~\ref{tab:hyper}). Removing the update from PromptShift-CRC lowers its coverage to \synNoUpdCov{} and raises its post-shift violation from \synPSPost{} to \synNoUpdPost{}. Marginal coverage therefore does not separate the weighting schemes, and we compare them on post-shift violation, maximum local violation and sharpness.

\paragraph{PromptShift-CRC against the matched baseline.}
Adaptive CRC with a window equal to PromptShift-CRC's recent window has post-shift violation \synMatchedPost{}, against \synPSPost{} for PromptShift-CRC. Paired by seed (Table~\ref{tab:paired}), PromptShift-CRC minus the matched baseline is between \synPostMatchedLo{} and \synPostMatchedHi{} across the five streams, and PromptShift-CRC is lower in only \synPostMatchedLowerMin{} to \synPostMatchedLowerMax{} of \nSeeds{} seeds, so it is slightly \emph{worse} on local validity. It is sharper: the difference in mean threshold is between \synThrMatchedLo{} and \synThrMatchedHi{}, and PromptShift-CRC is lower in \synThrMatchedLowerMin{} to \synThrMatchedLowerMax{} of \nSeeds{} seeds. The two differences have opposite signs and are of the same size, about half a percentage point, so the comparison is a trade-off and not a win.

\paragraph{What the components do.}
The gate buys validity at a cost in sharpness. PromptShift-CRC minus the variant with no gate has post-shift violation between \synPostNoGateLo{} and \synPostNoGateHi{} (lower in \synPostNoGateLowerMin{} to \synPostNoGateLowerMax{} of \nSeeds{} seeds) and threshold between \synThrNoGateLo{} and \synThrNoGateHi{}. The kernel weights do the opposite. The variant that keeps only the gate and the recency window has post-shift violation \synGateRecPost{} and mean threshold \synGateRecThr{}, against \synPSPost{} and \synPSThr{} for PromptShift-CRC. The kernel weights with the update but without the gate and the recency fallback have the sharpest thresholds among these variants (\synKernelUpdThr{}) and the highest post-shift violation (\synKernelUpdPost{}).

\paragraph{Frontier gap.}
A fairer comparison asks whether a method does better than Adaptive CRC at the same sharpness. Table~\ref{tab:frontier_syn} gives the frontier gap. For PromptShift-CRC it ranges from \PSfrontierSynLo{} to \PSfrontierSynHi{} over the five streams (standard error at most \PSfrontierSynSEmax{}) and is negative on \PSfrontierSynBelowCount{} of the five streams. No stream shows more than about half a percentage point of post-shift violation recovered by the drift-aware components.

\paragraph{Stronger online methods.}
DtACI \citep{gibbs2024online}, which runs several step sizes and weights them online, has post-shift violation \synDtACIPost{} on average, and PromptShift-CRC minus DtACI is between \synPostDtACILo{} and \synPostDtACIHi{} on the five streams, with PromptShift-CRC higher in every stream. Its threshold is also lower. Part of this comes from over-coverage: DtACI's coverage is \synDtACICov{} on average, above the target, which flatters its local violation, and the frontier gap is the fairer view. Quantile tracking is worse than every weighted rule that uses the update on these streams. The density-ratio rule is close to PromptShift-CRC. Among the online methods we tested, the handling of the step size matters about as much as the choice of weights.

\begin{figure}[t]
\centering
\includegraphics[width=\textwidth]{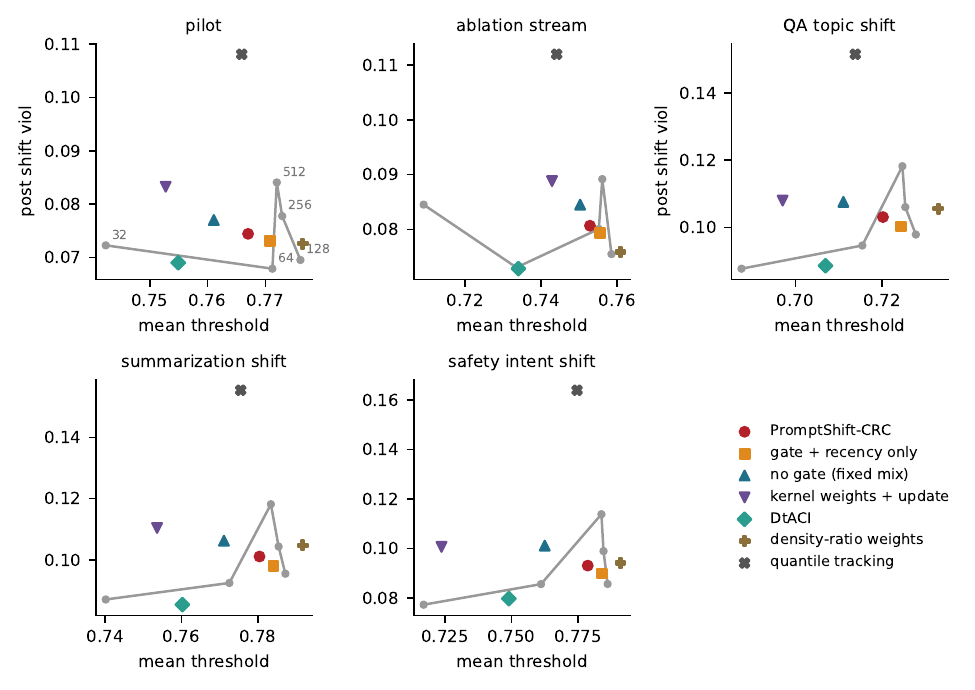}
\caption{Validity and sharpness on the five synthetic streams (mean over \nSeeds{} seeds). The grey curve is Adaptive CRC with windows $32$, $64$, $128$, $256$, $512$ (labelled on the first panel); points are other methods. Lower is better on both axes.}
\label{fig:frontier-syn}
\end{figure}


\subsection{Real model outputs}\label{sec:res-real}

We ran three open models on CPU and recorded their outputs (Section~\ref{sec:protocol}): Qwen2.5-1.5B-Instruct for MMLU and TruthfulQA, GPT-2 for RealToxicityPrompts, and Qwen2.5-0.5B for the replication. The risk is $1-p(\text{correct})$ for the two question-answering tasks and the toxicity probability of a continuation for the third. The tasks differ in one way that matters for this paper: how much the \emph{upper tail} of the risk moves between regimes, because the threshold $q_t$ targets the $0.9$ quantile.

\paragraph{MMLU and TruthfulQA: the upper tail does not move.}
Qwen2.5-1.5B-Instruct answers \MmluAcc{} of the \MmluN{} MMLU questions correctly. Across the four broad categories the median risk moves a great deal (from \MmluMedLo{} to \MmluMedHi{}) but the $0.9$ quantile does not (\MmluQnineLo{} to \MmluQnineHi{}). A threshold at the $0.9$ quantile is therefore close to $1$ in every regime, and the methods tie (Table~\ref{tab:mmlu_catblock}). In the category-block orderings static calibration already has post-shift violation \MmluCatStaticPost{}, against \MmluCatPSPost{} for PromptShift-CRC and \MmluCatMatchedPost{} for matched Adaptive CRC; the mean thresholds of the online methods lie within $0.015$ of one another; and the intervention rate is $0.999$ or higher for every method, so the decision rule almost never lets the model answer. PromptShift-CRC minus matched Adaptive CRC is \MmluCatPSvsMatchedPost{} (SE \MmluCatPSvsMatchedPostSE{}) in post-shift violation.

The frontier gaps are \MmluCatFrontierPS{} (SE \MmluCatFrontierPSSE{}) for PromptShift-CRC and \MmluCatFrontierDtACI{} (SE \MmluCatFrontierDtACISE{}) for DtACI in the category-block orderings, and \MmluSubjFrontierPS{} and \MmluSubjFrontierDtACI{} in the subject-block orderings. In \MmluCatFrontierPSOutN{} of the \MmluCatFrontierPSAllN{} category-block orderings the mean threshold of PromptShift-CRC lies outside the range of the window sweep, where the frontier is clipped to its nearest end; over the \MmluCatFrontierPSInN{} orderings inside the range the PromptShift-CRC gap is \MmluCatFrontierPSIn{} (SE \MmluCatFrontierPSInSE{}), against \MmluCatFrontierPS{} over all \MmluCatFrontierPSAllN{}.

On TruthfulQA the model answers \TqaAcc{} of the questions correctly under multiple-choice scoring, the $0.9$ quantile of the risk is \TqaQnine{}, and the same holds. At this level these two tasks test the guarantee (every online method reaches the target) but not the weighting, because no weighting can help when the quantity being estimated barely changes. We report them so that the reader can see the setting in which drift-awareness has nothing to offer.

\paragraph{MMLU at levels where the target quantile moves.}
The target quantile moves more at higher levels. For the category blocks the range over the four categories of the $(1-\alpha)$ quantile of the risk is \MmluSpreadAOne{} at $\alpha=0.1$, \MmluSpreadATwo{} at $\alpha=0.2$ and \MmluSpreadAThree{} at $\alpha=0.3$ (the $0.7$ quantile runs from \MmluSevenLo{} to \MmluSevenHi{}), so MMLU at $\alpha=0.3$ is a second real task on which the target quantile moves. We ran the category blocks at these two further levels with every constant as declared, and wrote down before the runs that any effect of the drift-aware components should grow with the movement of the target. It does not (Table~\ref{tab:mmlu_robust}). PromptShift-CRC minus the matched baseline in post-shift violation is \MmluATwoPSvsMatchedPost{} (SE \MmluATwoPSvsMatchedPostSE{}) at $\alpha=0.2$ and \MmluAThreePSvsMatchedPost{} (SE \MmluAThreePSvsMatchedPostSE{}) at $\alpha=0.3$, and the frontier gap of PromptShift-CRC is \MmluATwoFrontierPS{} (SE \MmluATwoFrontierPSSE{}) and \MmluAThreeFrontierPS{} (SE \MmluAThreeFrontierPSSE{}), against \MmluCatFrontierPS{} at $\alpha=0.1$. At $\alpha=0.3$ the gap over the \MmluAThreeFrontierPSInN{} orderings inside the range of the window sweep is \MmluAThreeFrontierPSIn{} (SE \MmluAThreeFrontierPSInSE{}). The criterion we had fixed for claiming that the components help at a moving target (a gap of at most $-0.005$, and more than two standard errors, at $\alpha=0.3$) is not met. As on RealToxicityPrompts, the components move post-shift violation by less than half a percentage point at equal sharpness, so a moving target is not sufficient for these weights to help.

\paragraph{RealToxicityPrompts: the upper tail moves.}
We drew \RtpN{} prompts, continued each with GPT-2 (one nucleus sample of $20$ tokens), and scored the continuation with a public toxicity classifier. Only \RtpFracHigh{} of the continuations have risk above $0.5$, and they are concentrated in prompts that are themselves toxic. Split by quartiles of the prompt toxicity annotation, the $0.9$ quantile of the risk is \RtpQnineLo{} in the lowest band and \RtpQnineHi{} in the highest, so a change of band is a real shift of the target quantile. The ordering uses only the annotation of the prompt, not the model's output. In the three lowest bands the $0.9$ quantile lies below $0.06$, close to the classifier's noise floor, so the intervention rate at $\tau=0.5$ carries little information on this task and we do not interpret it.

Static calibration fails here (coverage \RtpStaticCov{}, post-shift violation \RtpStaticPost{}), and rolling calibration is better but still above target after shifts (\RtpRollingPost{}). All online methods reach the target coverage (Table~\ref{tab:rtp_bandblock}, Figure~\ref{fig:frontier-rtp}). Among them the differences are small. Post-shift violation is \RtpPSPost{} for PromptShift-CRC, \RtpMatchedPost{} for the matched baseline, \RtpAdaptWPost{} for Adaptive CRC with the window $W$ and \RtpAdaptLongPost{} with the longest window of the sweep. Paired by ordering, PromptShift-CRC minus the matched baseline is \RtpPSvsMatchedPost{} (SE \RtpPSvsMatchedPostSE{}, PromptShift-CRC lower in \RtpPSvsMatchedPostLower{} of \RtpPSvsMatchedPostN{} orderings). On this stream a longer window works about as well as PromptShift-CRC, which is the opposite of the synthetic streams, where short windows were best. The frontier gap is \RtpFrontierPS{} (SE \RtpFrontierPSSE{}; \RtpFrontierPSIn{}, SE \RtpFrontierPSInSE{}, over the \RtpFrontierPSInN{} of \RtpFrontierPSAllN{} orderings inside the range of the sweep) for PromptShift-CRC, \RtpFrontierKernelUpd{} (SE \RtpFrontierKernelUpdSE{}) for the kernel weights with the update, and \RtpFrontierDtACI{} (SE \RtpFrontierDtACISE{}) for DtACI. The kernel with the median-distance bandwidth barely localises on sentence embeddings (the effective size of the kernel rule is close to its pool length), so its behaviour is close to that of uniform weights; this is a property of the rule with that bandwidth and we did not tune it.

\paragraph{Replication with a second generator and classifier.}
We repeated the experiment on the same \RtpN{} prompts and the same bands, with continuations from Qwen2.5-0.5B scored by a different classifier (\texttt{unitary/unbiased-toxic-roberta}). The risks of the two runs are only weakly related (correlation \RtpQCorr{}), because the continuations differ. Here \RtpQFracHigh{} of the continuations have risk above $0.5$, and the $0.9$ quantile of the risk is \RtpQQnineLo{} in the lowest band and \RtpQQnineHi{} in the highest, so the upper tail moves again. The results agree with those above (Table~\ref{tab:rtp_robust}). Static calibration fails (coverage \RtpQStaticCov{}, post-shift violation \RtpQStaticPost{}); post-shift violation is \RtpQPSPost{} for PromptShift-CRC, \RtpQMatchedPost{} for the matched baseline, \RtpQAdaptWPost{} for Adaptive CRC with the window $W$ and \RtpQKernelUpdPost{} for the kernel weights with the update. PromptShift-CRC minus the matched baseline is \RtpQPSvsMatchedPost{} (SE \RtpQPSvsMatchedPostSE{}, lower in \RtpQPSvsMatchedPostLower{} of \RtpQPSvsMatchedPostN{} orderings) and its frontier gap is \RtpQFrontierPS{} (SE \RtpQFrontierPSSE{}); the frontier gap of DtACI is \RtpQFrontierDtACI{} (SE \RtpQFrontierDtACISE{}). As on the first generator, a longer window does as well as PromptShift-CRC.

\paragraph{Other target levels.}
Running the GPT-2 streams at $\alpha=0.05$ and $\alpha=0.2$ gives the same pattern (Table~\ref{tab:rtp_robust}). PromptShift-CRC minus the matched baseline in post-shift violation is \RtpAlphaLowPSvsMatchedPost{} (SE \RtpAlphaLowPSvsMatchedPostSE{}) at $\alpha=0.05$ and \RtpAlphaHighPSvsMatchedPost{} (SE \RtpAlphaHighPSvsMatchedPostSE{}) at $\alpha=0.2$; at $\alpha=0.2$ the difference is not distinguishable from zero. The frontier gap of PromptShift-CRC is \RtpAlphaLowFrontierPS{} (SE \RtpAlphaLowFrontierPSSE{}) at $\alpha=0.05$ and \RtpAlphaHighFrontierPS{} (SE \RtpAlphaHighFrontierPSSE{}) at $\alpha=0.2$.

\begin{figure}[t]
\centering
\includegraphics[width=\textwidth]{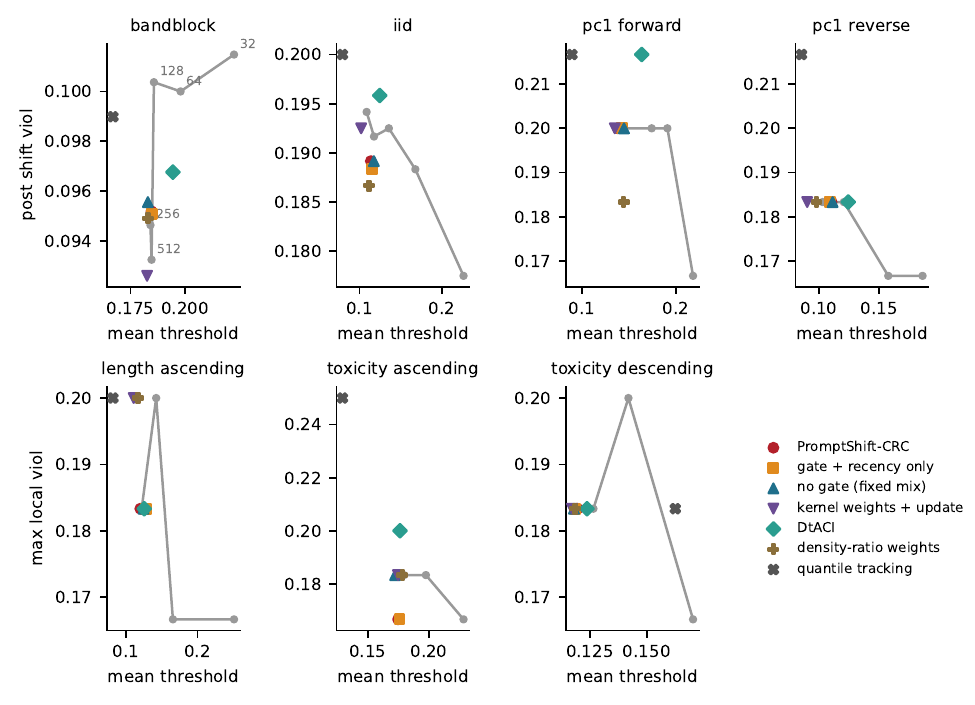}
\caption{Validity and sharpness on RealToxicityPrompts with GPT-2 continuations. Panels are ordering families (each point is the mean over the orderings in the family). The grey curve is Adaptive CRC over the window sweep; labels give the window.}
\label{fig:frontier-rtp}
\end{figure}

\subsection{Robustness to the window choice}\label{sec:res-window}

The results suggest that the value of PromptShift-CRC is that no window has to be chosen. The best window for Adaptive CRC depends on the stream. It is short on the synthetic streams and long on RealToxicityPrompts, and the window that is best in hindsight is known only afterwards. We state the comparison in two ways.

\emph{Rank.} Among the five Adaptive CRC windows and PromptShift-CRC, PromptShift-CRC has mean rank \regSynRank{} on the synthetic streams, \regMmluCatRank{} and \regMmluSubjRank{} on MMLU, \regRtpRank{} on RealToxicityPrompts and \regTqaRank{} on TruthfulQA (Table~\ref{tab:regret}; $1$ is the lowest violation and $3.5$ is the middle of six). It is in the better half on MMLU and RealToxicityPrompts, and not on the synthetic streams or on TruthfulQA. In every setting it is better than the worst window, by between two-fifths and two-thirds of the spread of the family.

\emph{Choosing the window elsewhere.} The best window in hindsight is selected on the data it is scored on, which biases the comparison against PromptShift-CRC. A fairer protocol chooses the window on the other families of streams and applies it to the held-out family (Table~\ref{tab:lofo}). PromptShift-CRC minus the chosen window is \lofosynthetic{} (synthetic), \lofoMMLUcategories{} and \lofoMMLUsubjects{} (MMLU), \lofoRTPbands{} (RealToxicityPrompts) and \lofoTruthfulQAcategories{} (TruthfulQA). It is lower in four of the five families (Table~\ref{tab:lofo} gives standard errors): clearly so on the synthetic streams and on RealToxicityPrompts, within two standard errors on the two MMLU families, and it loses on TruthfulQA (\lofoTruthfulQAcategories{}, SE \lofoTruthfulQAcategoriesSE{}). The best window changes from the shortest (synthetic) to the longest (RealToxicityPrompts), and that change is the argument for a rule that does not choose a window. We do not claim that PromptShift-CRC beats a well-chosen window; the evidence is that it avoids a poorly chosen one.

\subsection{What each component does}\label{sec:res-ablation}

On the synthetic streams and on RealToxicityPrompts the same pattern holds (Tables~\ref{tab:synthetic} and~\ref{tab:rtp_bandblock}). The online update is the component that matters most: removing it moves post-shift violation from \synPSPost{} to \synNoUpdPost{} on the synthetic streams and from \RtpPSPost{} to \RtpNoUpdPost{} on RealToxicityPrompts. The recency window with the gate is next: the variant that keeps only the gate and the recency window has post-shift violation \synGateRecPost{} (synthetic) and \RtpGateRecPost{} (RealToxicityPrompts), close to PromptShift-CRC's \synPSPost{} and \RtpPSPost{}. The kernel weights contribute least. With them alone and the update, post-shift violation is \synKernelUpdPost{} on the synthetic streams, higher than PromptShift-CRC's, with a sharper threshold. These components move post-shift violation by less than one percentage point and the mean threshold by up to $1.6$ points, and the sign depends on the stream.

Table~\ref{tab:mechanism} puts every real setting with a post-shift analysis on one axis, the range over regimes of the target quantile of the risk. In the \MechRows{} settings the range runs from \MechSpreadMin{} to \MechSpreadMax{}, and the frontier gap of PromptShift-CRC is at most \MechGapMax{} in absolute value in every one of them, whether it is taken over all orderings or over those inside the range of the window sweep. The largest differences from the matched baseline are on RealToxicityPrompts, where PromptShift-CRC minus the matched baseline in post-shift violation is between \RtpAlphaHighPSvsMatchedPost{} and \RtpPSvsMatchedPost{} at the three GPT-2 levels, but a longer window does as well, so these are not gains over the frontier. So the data do not show a setting in which the drift-aware components recover more than a fraction of a percentage point of post-shift violation at equal sharpness, whatever the movement of the target quantile.

\subsection{Does the computed drift predict violations?}\label{sec:res-trace}

An earlier draft described the post-gate drift $D_t$ as a useful diagnostic. We tested whether the pre-gate drift, which does not depend on the scores, predicts violations, with three statistics fixed before the traces were analysed (Table~\ref{tab:trace}): the violation rate in the steps of a stream whose drift is in its top decile minus the rate elsewhere, the AUC of the drift for the violation indicator, and, across the shifts of a family, the Spearman correlation between the jump in drift at a boundary and the post-shift violation. We had fixed the criterion for success as a top-decile difference of at least $0.02$ (and more than two standard errors) or an AUC of at least $0.55$, on at least one real family.

The top-decile criterion is met on one of the three real families, MMLU categories (\TraceMmluTop{}, SE \TraceMmluTopSE{}). The AUC criterion is met on none (\TraceMmluAuc{} on MMLU, \TraceRtpAuc{} and \TraceRtpQAuc{} on RealToxicityPrompts). By the rule we fixed in advance, which required either criterion on one real family, the drift therefore counts as predictive. We do not describe it as an alarm, because the sign reverses on RealToxicityPrompts, where the AUC is also below $0.5$, a case the rule did not anticipate. On RealToxicityPrompts the sign reverses: the steps with the highest drift have fewer violations than the others (\TraceRtpTop{} with GPT-2 and \TraceRtpQTop{} with Qwen2.5-0.5B). On all three real families the shift-level correlation is negative (\TraceRealRhoLo{} to \TraceRealRhoHi{}), so the larger the jump in drift at a boundary, the lower the violation rate after it. On the synthetic streams the step-level association is positive (\TraceSynTopLo{} to \TraceSynTopHi{}), but the shift-level correlation changes sign from stream to stream. The matched baseline gives the same signs. On the real streams the pre-gate drift is almost flat across a boundary (on RealToxicityPrompts \TraceRtpDBefore{} in the 20 steps before and \TraceRtpDAfter{} in the 25 steps after), and the violation rate in the first 25 steps after a boundary is \TraceRtpErrFirst{} against \TraceRtpErrLater{} in the next 125 steps on RealToxicityPrompts, with no excess on MMLU (\TraceMmluErrFirst{} and \TraceMmluErrLater{}). We therefore describe the drift as a measure of how unlike its pool a prompt is, and not as an alarm for validity.

\subsection{A published localised competitor: L-ARC}\label{sec:res-larc}

Localised adaptive risk control (L-ARC, \citeauthor{zecchin2024localized}~\citeyear{zecchin2024localized}) is the published method closest to PromptShift-CRC: an online threshold that depends on the features through a kernel. We ran it on every stream of the main tables (Table~\ref{tab:larc}). The update and its constants are those of the authors; the features are the same unit-norm embeddings as for every other method, with no rescaling, and the threshold starts from the calibration quantile, as it does for quantile tracking. Before the runs we fixed the authors' defaults (step size $\eta_1=1$, kernel length scale $1$) as the primary configuration and decided to report three other configurations in full (a step size of $0.1$, and a length scale matched to our bandwidth; Table~\ref{tab:larc_all}). We claim no coverage guarantee for L-ARC in this setting.

With its defaults L-ARC holds the target on the synthetic streams (coverage \LarcSynCovLo{} to \LarcSynCovHi{}) but lies above the Adaptive CRC frontier by \LarcSynLo{} to \LarcSynHi{} in post-shift violation on all five streams, where PromptShift-CRC is within \PSfrontierMaxAbs{}. With the smaller step it does not reach the target (coverage \LarcSynSmallCovLo{} to \LarcSynSmallCovHi{}) and lies above the frontier by \LarcSynSmallLo{} to \LarcSynSmallHi{}. On MMLU and TruthfulQA it is on the frontier, as every method is (gaps \LarcMmluCatGap{} for category blocks, \LarcMmluSubjGap{} for subject blocks and \LarcTqaGap{} for TruthfulQA). On RealToxicityPrompts it is below the frontier with the GPT-2 continuations (\LarcRtpGap{}, SE \LarcRtpGapSE{}; \LarcRtpGapIn{} over the \LarcRtpGapInN{} orderings inside the range of the sweep) and above it with the Qwen2.5-0.5B continuations (\LarcRtpQGap{}, SE \LarcRtpQGapSE{}), where the gaps of PromptShift-CRC are \RtpFrontierPS{} and \RtpQFrontierPS{}. L-ARC minus PromptShift-CRC in post-shift violation is \LarcRtpMinusPS{} (SE \LarcRtpMinusPSSE{}) and \LarcRtpQMinusPS{} (SE \LarcRtpQMinusPSSE{}) on the two runs, and in mean threshold \LarcRtpThr{} and \LarcRtpQThr{}.

The two sentences we had written in advance for this comparison do not fit the outcome, because on the GPT-2 run both L-ARC and PromptShift-CRC are at or below the frontier. We therefore state it directly. With its authors' defaults L-ARC is not better than PromptShift-CRC against the frontier in a consistent way: it is well above the frontier on the synthetic streams, on it on MMLU and TruthfulQA, below it on one RealToxicityPrompts run and above it on the other.

\subsection{Do the weights help when the kernel localises?}\label{sec:res-bw}

The kernel at the median bandwidth is close to uniform (Section~\ref{sec:res-real}), and every block ordering above visits each regime once, so a past prompt like the current one is rarely in the pool when it would matter. We tested both limits with two designs fixed in advance (Table~\ref{tab:bandwidth}). First, we repeated the band-block and category-block runs with the kernel bandwidth set to $h/4$, $h/2$, $h$ and $2h$, with the gate centre and slope rescaled with it, and report every setting. Second, we built orderings in which blocks of $128$ prompts cycle through the regimes (seven visits to each toxicity band, four to each MMLU category), so that the previous visit to a regime is still in the pool of $768$ steps, and ran the same four bandwidths on them: on RealToxicityPrompts with both generators and on MMLU at $\alpha=0.1$, $0.2$ and $0.3$.

The sharper kernels do localise. On RealToxicityPrompts with GPT-2 the effective size of the kernel-plus-update rule falls from \BwKUNeffOneRtp{} at $h$ to \BwKUNeffQuarterRtp{} at $h/4$, and on MMLU from \BwKUNeffOneMmlu{} to \BwKUNeffQuarterMmlu{}; the effective size of PromptShift-CRC falls from \BwPSNeffOne{} to \BwPSNeffQuarter{}, because its gate then falls back on the recent window. In the block orderings the frontier gap of PromptShift-CRC is between \BwPSSweepLo{} and \BwPSSweepHi{} across the four bandwidths and three settings. In the recurring-regime orderings it is within \BwPSRecurMaxGap{} of zero in every setting at every bandwidth, and PromptShift-CRC is within \BwPSRecurMaxMatched{} of the matched baseline.

The numerical criterion we had fixed for a positive finding (a frontier gap below $-0.005$ and more than two standard errors, for a kernel variant at a bandwidth) can be read over all $24$ orderings or over those inside the range of the window sweep, and the text we wrote in advance did not say which. Over all orderings it is met in \BwPass{} of the \BwCells{} cells. Both are the same rule, kernel weights with the update (the variant without the recency fallback coincides with it), at $h/4$ on the GPT-2 block orderings (\BwPassGap{}, SE \BwPassSE{}). In none of those $24$ orderings does the mean threshold lie inside the range of the sweep: it is \BwPassThr{}, against \BwShortThr{} for the shortest window and \BwMatchedThr{} for the matched baseline. The rule is much blunter than any window, so its gap is measured against the end of the curve and is not a comparison at equal sharpness. It does not reappear with the Qwen2.5-0.5B continuations (\BwKUQuarterQwenGap{}, SE \BwKUQuarterQwenSE{}), and on MMLU the same variant has the opposite sign (\BwKUQuarterMmluGap{}, SE \BwKUQuarterMmluSE{}).

Over the orderings inside the range, in the cells that have at least five of them, the criterion is met in \BwInsidePass{} of \BwInsideCells{} cells. They are the same variant, the no-gate rule (a fixed half-and-half mix of the kernel weights and the recent window) at $h/4$, on the two RealToxicityPrompts runs: \BwLeadGptTwoIn{} (SE \BwLeadGptTwoInSE{}) over \BwLeadGptTwoInN{} of $24$ orderings with GPT-2, and \BwLeadQwenIn{} (SE \BwLeadQwenInSE{}) over \BwLeadQwenInN{} of $24$ with Qwen2.5-0.5B. Over all $24$ orderings the gaps are \BwLeadGptTwoAll{} (SE \BwLeadGptTwoAllSE{}) and \BwLeadQwenAll{} (SE \BwLeadQwenAllSE{}). This is the only configuration in the study that lies about half a percentage point below the frontier, and we regard it as a lead to test, not as a finding; this is a judgment made after the result, because the rule we fixed in advance had no correction for the \BwInsideCells{} comparisons. It is two of \BwInsideCells{} comparisons, and at a threshold of two standard errors about three cells would be expected by chance if no variant helped, although the cells are correlated. The two runs use the same prompts and orderings. It does not appear on MMLU or in the recurring-regime orderings. It was not the rule we proposed. In the other cells with at least five orderings inside the range, the gap is smaller than half a percentage point in absolute value.

DtACI, which has no kernel, lies clearly below the frontier on the recurring-regime orderings of RealToxicityPrompts, where its gap is \BwDtaciRecurGptTwo{} (SE \BwDtaciRecurGptTwoSE{}) with GPT-2 and \BwDtaciRecurQwen{} (SE \BwDtaciRecurQwenSE{}) with Qwen2.5-0.5B, with all $24$ orderings inside the range of the sweep.

\begin{table}[t]
\centering
\small
\caption{Synthetic streams, mean over five generators of the 30-seed means (and mean seed-to-seed sd). Every online method has coverage near $0.90$ because of the update; local violation and sharpness separate the methods.}
\label{tab:synthetic}
\begin{tabular}{lrrrr}
\toprule
Method & Coverage & Post-shift violation & Max local violation & Mean threshold \\
\midrule
Static & 0.334 $\pm$ 0.012 & 0.802 $\pm$ 0.016 & 0.986 & 0.450 \\
Rolling & 0.854 $\pm$ 0.005 & 0.311 $\pm$ 0.015 & 0.524 & 0.758 \\
Adaptive CRC, window $W$ & 0.901 $\pm$ 0.001 & 0.095 $\pm$ 0.005 & 0.201 & 0.764 \\
Adaptive CRC, matched window & 0.901 $\pm$ 0.001 & 0.088 $\pm$ 0.006 & 0.188 & 0.764 \\
Kernel weights + update & 0.903 $\pm$ 0.001 & 0.098 $\pm$ 0.005 & 0.189 & 0.734 \\
Kernel weights, no update & 0.860 $\pm$ 0.007 & 0.295 $\pm$ 0.019 & 0.527 & 0.737 \\
PromptShift-CRC & 0.901 $\pm$ 0.001 & 0.090 $\pm$ 0.005 & 0.186 & 0.760 \\
Gate + recency only & 0.901 $\pm$ 0.001 & 0.088 $\pm$ 0.006 & 0.187 & 0.764 \\
No gate (fixed mix) & 0.902 $\pm$ 0.001 & 0.095 $\pm$ 0.005 & 0.186 & 0.751 \\
No recency fallback & 0.903 $\pm$ 0.001 & 0.098 $\pm$ 0.005 & 0.189 & 0.734 \\
No online update & 0.879 $\pm$ 0.004 & 0.214 $\pm$ 0.015 & 0.387 & 0.762 \\
DtACI & 0.909 $\pm$ 0.003 & 0.079 $\pm$ 0.007 & 0.158 & 0.741 \\
Density-ratio weights & 0.902 $\pm$ 0.001 & 0.091 $\pm$ 0.006 & 0.198 & 0.771 \\
Quantile tracking & 0.891 $\pm$ 0.001 & 0.138 $\pm$ 0.005 & 0.231 & 0.755 \\
\bottomrule
\end{tabular}
\par\smallskip{\footnotesize Post-shift violation is the miss rate over the 150 steps after each regime boundary (target $0.10$).}
\end{table}

\begin{table}[t]
\centering
\small
\caption{Frontier gap on the synthetic streams: post-shift violation minus the Adaptive-CRC window-sweep curve interpolated at the method's own mean threshold, mean over seeds with standard error. Negative means better than Adaptive CRC at the same sharpness.}
\label{tab:frontier_syn}
\resizebox{\textwidth}{!}{%
\begin{tabular}{lrrrrr}
\toprule
Stream & PromptShift-CRC & Gate + recency only & Kernel + update & DtACI & Density ratio \\
\midrule
pilot & +0.0015 (0.0011) & -0.0023 (0.0009) & +0.0111 (0.0010) & -0.0032 (0.0013) & -0.0002 (0.0012) \\
ablation stream & -0.0013 (0.0008) & -0.0033 (0.0008) & +0.0113 (0.0010) & -0.0015 (0.0014) & -0.0030 (0.0013) \\
QA topic shift & +0.0015 (0.0011) & -0.0032 (0.0010) & +0.0179 (0.0013) & -0.0043 (0.0016) & +0.0028 (0.0016) \\
summarization shift & -0.0049 (0.0019) & -0.0054 (0.0011) & +0.0212 (0.0011) & -0.0050 (0.0013) & +0.0045 (0.0017) \\
safety-intent shift & -0.0055 (0.0014) & -0.0049 (0.0014) & +0.0218 (0.0009) & -0.0037 (0.0014) & -0.0010 (0.0019) \\
\bottomrule
\end{tabular}
}
\end{table}

\begin{table}[t]
\centering
\small
\caption{Frontier gap on MMLU streams (post-shift violation minus the Adaptive-CRC window-sweep curve at the same mean threshold; mean over orderings with standard error). Negative is better than Adaptive CRC. Where a method's mean threshold lies outside the range of the window sweep the curve is clipped to its nearest end; the last column repeats the PromptShift-CRC gap over the orderings whose threshold lies inside the range.}
\label{tab:frontier_mmlu}
\resizebox{\textwidth}{!}{%
\begin{tabular}{lrrrrrr}
\toprule
Ordering & PromptShift-CRC & Gate + recency only & Kernel + update & DtACI & Density ratio & PromptShift-CRC, inside range \\
\midrule
cat blocks & -0.0007 (0.0006) & -0.0007 (0.0006) & -0.0007 (0.0006) & -0.0084 (0.0012) & -0.0006 (0.0006) & +0.0001 (0.0006), 17 of 24 \\
subject blocks & -0.0001 (0.0001) & -0.0001 (0.0001) & +0.0002 (0.0001) & -0.0060 (0.0006) & +0.0001 (0.0002) & -0.0002 (0.0001), 18 of 20 \\
\bottomrule
\end{tabular}
}
\end{table}

\begin{table}[t]
\centering
\small
\caption{Frontier gap on RTP streams (post-shift violation minus the Adaptive-CRC window-sweep curve at the same mean threshold; mean over orderings with standard error). Negative is better than Adaptive CRC. Where a method's mean threshold lies outside the range of the window sweep the curve is clipped to its nearest end; the last column repeats the PromptShift-CRC gap over the orderings whose threshold lies inside the range.}
\label{tab:frontier_rtp}
\resizebox{\textwidth}{!}{%
\begin{tabular}{lrrrrrr}
\toprule
Ordering & PromptShift-CRC & Gate + recency only & Kernel + update & DtACI & Density ratio & PromptShift-CRC, inside range \\
\midrule
band blocks & -0.0010 (0.0010) & -0.0007 (0.0011) & -0.0027 (0.0012) & -0.0030 (0.0015) & +0.0006 (0.0008) & -0.0012 (0.0013), 19 of 24 \\
\bottomrule
\end{tabular}
}
\end{table}

\begin{table}[t]
\centering
\small
\caption{Robustness to the window choice. For each stream, Adaptive CRC is run with a sweep of windows. PromptShift-CRC (PS) has no window to choose. Entries are means over streams of post-shift violation: PS minus the best window in hindsight, the spread of the window family (worst minus best), PS minus the worst window, and the rank of PS among the window family and PS (1 = lowest violation). The best window in hindsight is selected on the same data, so the first column is biased in PS's disfavour.}
\label{tab:regret}
\resizebox{\textwidth}{!}{%
\begin{tabular}{lrrrrr}
\toprule
Setting & Streams & PS $-$ best & Spread & PS $-$ worst & Rank of PS \\
\midrule
synthetic (5 streams x 30 seeds) & 150 & +0.0124 & 0.0267 & -0.0143 & 3.9 of 6 \\
MMLU, category blocks & 24 & +0.0039 & 0.0100 & -0.0061 & 3.2 of 6 \\
MMLU, subject blocks & 20 & +0.0007 & 0.0020 & -0.0013 & 3.2 of 6 \\
RealToxicityPrompts, toxicity bands & 24 & +0.0046 & 0.0144 & -0.0097 & 3.0 of 6 \\
TruthfulQA, category blocks & 20 & +0.0051 & 0.0090 & -0.0038 & 3.9 of 6 \\
RealToxicityPrompts replication (Qwen2.5-0.5B) & 24 & +0.0043 & 0.0149 & -0.0106 & 3.0 of 6 \\
\bottomrule
\end{tabular}
}
\end{table}

\begin{table}[t]
\centering
\small
\caption{Choosing the Adaptive CRC window without looking at the held-out family. For each family of streams, the window (position in the five-window sweep, short to long) with the best mean rank over the other four families is chosen; the entry is PromptShift-CRC's post-shift violation minus that window's on the held-out family, with the paired standard error over streams or seeds (negative: PromptShift-CRC is lower). The last two columns give the window that is best on the held-out family itself and PromptShift-CRC's gap to that window in hindsight.}
\label{tab:lofo}
\begin{tabular}{lrrrr}
\toprule
Held-out family & Window chosen elsewhere & PS $-$ chosen & Best here & PS $-$ best here \\
\midrule
synthetic & 4 of 5 & -0.0030 (0.0003) & 1 of 5 & +0.0124 \\
MMLU categories & 2 of 5 & -0.0019 (0.0011) & 4 of 5 & +0.0039 \\
MMLU subjects & 2 of 5 & -0.0006 (0.0004) & 4 of 5 & +0.0007 \\
RTP bands & 2 of 5 & -0.0048 (0.0015) & 5 of 5 & +0.0046 \\
TruthfulQA categories & 4 of 5 & +0.0014 (0.0006) & 2 of 5 & +0.0051 \\
\bottomrule
\end{tabular}
\end{table}

\begin{table}[t]
\centering
\small
\caption{How far the target quantile moves between regimes, and what the drift-aware components recover. The spread is the range over regimes with at least 30 prompts of the $(1-\alpha)$ quantile of the risk (a property of the data and the regime labels); the last three columns are PromptShift-CRC minus the matched baseline in post-shift violation, the frontier gap, and the frontier gap over the orderings inside the range of the window sweep.}
\label{tab:mechanism}
\resizebox{\textwidth}{!}{%
\begin{tabular}{lrrrrr}
\toprule
Setting & Level $\alpha$ & Spread of target quantile & PS $-$ matched post-shift & Frontier gap PS & PS, inside range \\
\midrule
RealToxicityPrompts, GPT-2 & 0.05 & 0.858 & -0.0039 (0.0009) & -0.0010 (0.0009) & -0.0017 (0.0009) \\
RealToxicityPrompts, GPT-2 & 0.1 & 0.579 & -0.0052 (0.0011) & -0.0010 (0.0010) & -0.0012 (0.0013) \\
RealToxicityPrompts, GPT-2 & 0.2 & 0.094 & -0.0031 (0.0019) & -0.0017 (0.0015) & -0.0020 (0.0015) \\
RealToxicityPrompts, Qwen2.5-0.5B & 0.1 & 0.442 & -0.0050 (0.0015) & -0.0003 (0.0007) & -0.0007 (0.0008) \\
MMLU & 0.1 & 0.025 & -0.0010 (0.0007) & -0.0007 (0.0006) & +0.0001 (0.0006) \\
MMLU & 0.2 & 0.077 & -0.0016 (0.0009) & +0.0003 (0.0007) & +0.0007 (0.0007) \\
MMLU & 0.3 & 0.242 & -0.0004 (0.0013) & +0.0002 (0.0012) & +0.0004 (0.0013) \\
TruthfulQA & 0.1 & 0.000 & +0.0013 (0.0009) & -0.0004 (0.0007) & -0.0004 (0.0007) \\
\bottomrule
\end{tabular}
}
\end{table}

\begin{table}[t]
\centering
\small
\caption{L-ARC with the authors' defaults against the Adaptive CRC frontier and against PromptShift-CRC (PS). Frontier gap: post-shift violation minus the window-sweep curve at the method's own mean threshold (negative is better), with the gap over the orderings inside the range of the sweep where there are enough of them. Last two columns: L-ARC minus PromptShift-CRC in post-shift violation (standard error) and in mean threshold. Synthetic: mean over the five streams of the 30-seed means. No coverage guarantee is claimed for L-ARC here.}
\label{tab:larc}
\resizebox{\textwidth}{!}{%
\begin{tabular}{lrrrrr}
\toprule
Setting & Gap of PS & Gap of L-ARC & L-ARC, inside range & L-ARC $-$ PS post-shift & L-ARC $-$ PS threshold \\
\midrule
synthetic (mean of the five streams) & -0.0017 & +0.0277 (0.0005) & -- & +0.0237 (0.0004) & -0.0069 \\
MMLU, category blocks & -0.0007 & +0.0001 (0.0010) & +0.0013 (0.0022), 8 of 24 & +0.0019 (0.0011) & -0.0091 \\
MMLU, subject blocks & -0.0001 & -0.0002 (0.0004) & -0.0002 (0.0004), 20 of 20 & +0.0002 (0.0002) & -0.0087 \\
RealToxicityPrompts, GPT-2 & -0.0010 & -0.0033 (0.0033) & -0.0109 (0.0028), 16 of 24 & +0.0011 (0.0027) & -0.0016 \\
RealToxicityPrompts, Qwen2.5-0.5B & -0.0003 & +0.0064 (0.0028) & -0.0079 (0.0050), 6 of 24 & +0.0064 (0.0027) & -0.0107 \\
TruthfulQA, category blocks & -0.0004 & -0.0008 (0.0011) & -- & -0.0017 (0.0013) & -0.0044 \\
\bottomrule
\end{tabular}
}
\end{table}


\section{Discussion}\label{sec:discussion}

\paragraph{What a practitioner should take from this.}
If the aim is long-run control of a risk threshold under prompt shift, use an online update with a calibration window of the right length. The update gives the guarantee (Theorem~\ref{thm:longrun}), and the best window differed between our streams. Among the online rules we tried, DtACI was the strongest on the synthetic streams, on MMLU and on the recurring-regime orderings of RealToxicityPrompts. PromptShift-CRC is a reasonable choice when the right window is unknown: in every setting it did better than the worst window of the Adaptive CRC family, and it was in the better half of the family on MMLU and RealToxicityPrompts. Its computed prompt drift $D_t$ and effective size $n_{\mathrm{eff},t}$ say how far the current prompt is from the evidence behind the threshold. We tested whether the drift predicts violations (Section~\ref{sec:res-trace}) and found that it does not do so reliably, so we describe it as a measure of novelty and not as an alarm.

\paragraph{Why the weights add little.}
Theorem~\ref{thm:longrun} gives the update full responsibility for the long-run rate. The weights can then only improve the \emph{local} rate and the sharpness. Two things limit what they can do. First, a kernel on sentence embeddings groups prompts by topic, and the risk of a prompt depends only weakly on its topic, so a topic-similar past prompt is a weak guide to the current risk; MMLU is the clearest example, since the median risk differs across categories while the upper tail does not. Second, down-weighting costs effective size (Lemma~\ref{lem:m}), and with a threshold at the $0.9$ quantile that cost is paid in the level of the quantile. On RealToxicityPrompts, where the upper tail does move between bands, a longer effective window helps a little. The kernel at the median bandwidth is close to uniform there: a prompt from the highest toxicity band receives \RtpKernelRatio{} times the kernel weight of a prompt from another band, and the effective size of the kernel rule is \RtpKernelNeff{} of a pool of \PoolLen{}. The gain is therefore a window effect and not a similarity effect. A kernel four times sharper does localise, but it does not move post-shift violation by more than about half a percentage point at equal sharpness, with one lead that we do not claim (Section~\ref{sec:res-bw}).

\section{Limitations}\label{sec:limits}

\begin{itemize}[nosep]
  \item \textbf{Scale of the real-output experiments.} We used three small open models (Qwen2.5-1.5B-Instruct, GPT-2 and Qwen2.5-0.5B) on CPU, two toxicity classifiers, one prompt template, one sampled continuation per prompt, and one embedding model. Larger models, other templates and other embeddings may behave differently. We generated toxic text only to measure its toxicity, and we release risk scores and not the text. The target quantile moves materially only on RealToxicityPrompts and on MMLU at $\alpha=0.3$, and neither shows a gain over the frontier; a task with a different kind of risk, such as summarisation faithfulness, would strengthen the conclusions.
  \item \textbf{The synthetic generators and their constants.} The generators and the per-experiment constants (bandwidth, gate centre and slope) are those of the earlier version of this work and were chosen while the generators were being developed, so they are not independent of the generators. We did not tune any constant after seeing the results reported here, and the constants on the real streams follow rules fixed in advance (Section~\ref{sec:protocol}).
  \item \textbf{Theory.} Theorem~\ref{thm:cond} needs conditional independence of the risks given the features and applies at a fixed level; Corollary~\ref{cor:drift} needs a Lipschitz condition that we cannot check. The long-run result (Theorem~\ref{thm:longrun}) holds for any weights but says nothing about local behaviour. Proposition~\ref{prop:running} bounds the conditional violation probability at the running level, but we have no guarantee that the running level stays near $\alpha$ over a short window, and the window identity in Remark~\ref{rem:running} shows that this is what the local violation rate measures. We do not assume that our streams satisfy Assumption~\ref{ass:ci}: they are orderings of a fixed item set with one recorded risk per item, so Theorem~\ref{thm:cond} and Proposition~\ref{prop:running} are statements about the rule, not certificates for these streams.
  \item \textbf{Proxy streams.} The streams built from TruthfulQA, RealToxicityPrompts, SummEval and LongHalQA fields in the earlier version contain no model outputs and some are ordered by a field linked to the outcome. We report them in Appendix~\ref{app:proxy} and draw no conclusion from them.
  \item \textbf{Uncertainty on the real streams.} Standard errors on the real streams are over orderings of one fixed set of items, so they measure ordering variability only, not variability across item sets or generations. The released risk scores are canonical: sampled continuations are not bit-reproducible across hardware.
  \item \textbf{The toxicity task.} The toxicity classifier is the ground truth, and $81\%$ of the risks lie below $0.01$, close to its floor; the $0.9$ quantile in the three lowest bands is in that range. Only abrupt block shifts have known boundaries, so post-shift violation is reported for them; smooth orderings (embedding direction, length) are reported with local violation only.
  \item \textbf{Weak localisation.} With the bandwidth set to the median pairwise distance, the kernel weights are close to uniform on all of the real streams (effective size \RtpKernelNeff{} of \PoolLen{} on RealToxicityPrompts and \MmluKernelNeff{} on MMLU), not only on RealToxicityPrompts, which limits what ``the kernel weights contribute least'' can mean. Sharper kernels do localise (Section~\ref{sec:res-bw}), and we report them in full. TruthfulQA multiple-choice accuracy of the small model is near chance, so that task is uninformative about weighting.
  \item \textbf{Baselines.} The density-ratio rule is a simplified stand-in for published methods, and L-ARC \citep{zecchin2024localized} was run with the authors' update and defaults but without their feature rescaling and with a warm start (Section~\ref{sec:res-larc}).
\end{itemize}

\section{Conclusion}\label{sec:conclusion}

We set out to show that drift-aware weighting improves conformal risk control for foundation models under prompt shift. What we found is narrower. The online update provides the long-run guarantee for any weights; the weights and the gate change local validity by less than one percentage point and sharpness by up to $1.6$ points; a matched short window and DtACI are as good or better in most of our settings; and PromptShift-CRC performs like a mid-ranked window without having to choose one. Tests fixed in advance to give the weights a better chance (a sharper kernel, regimes that recur, a second moving target, a published localised competitor) show no consistent gain over the frontier. We state a conditional bound with a computable prompt-drift term, give an evaluation protocol for comparing such rules fairly, and release the code, the model outputs and the streams so that others can test the drift-aware idea on tasks where the upper tail of the risk moves with the prompt.

\clearpage

\section*{Data and code availability}
The code, the recorded model outputs (risk scores, features and the identifiers needed to rebuild the streams), the result tables and the paper source are in the public repository \url{https://github.com/jopoku16/promptshift-crc}. Question and prompt text from MMLU and TruthfulQA can be re-downloaded from the public datasets; the RealToxicityPrompts prompt and continuation text is withheld from the release because it contains offensive language, and the generation script regenerates it. Sampled continuations are not bit-reproducible across hardware, so the released risk scores are canonical. Appendix~\ref{app:repro} gives the commands.

\printbibliography

\appendix
\section{Weighted quantile uncertainty}\label{app:ess}
\begin{proposition}[Weighted quantile uncertainty]\label{prop:ess}
Conditional on $E_{1:t}$, suppose $R_1,\dots,R_{t-1}$ are independent with a common continuous distribution function $F$ (the no-drift case). For any
$\eta\in(0,1)$ and integer $M\ge 1$, with probability at least $1-\eta$,
\[
 \sup_{x}\Bigl|\sum_{s<t}p_{t,s}\mathbf 1\{R_s\le x\}-F(x)\Bigr|\ \le\ \sqrt{\frac{\log(2M/\eta)}{2\,n_{\mathrm{eff},t}}}+\frac1M .
\]
\end{proposition}

\begin{proof}
Let $u_j=j/M$ and $x_j=F^{-1}(u_j)$ for $j=1,\dots,M-1$, and $x_0=-\infty$, $x_M=+\infty$; by continuity
$F(x_j)=u_j$. Let $\hat F(x)=\sum_s p_{t,s}\mathbf 1\{R_s\le x\}$. For fixed $x_j$, $\hat F(x_j)-F(x_j)$ is a
sum of independent terms $p_{t,s}(\mathbf 1\{R_s\le x_j\}-F(x_j))$, each in an interval of length $p_{t,s}$, so Hoeffding's
inequality gives $\mathbb P(|\hat F(x_j)-F(x_j)|>\varepsilon)\le 2\exp(-2\varepsilon^2/\sum_s p_{t,s}^2)=2\exp(-2\varepsilon^2n_{\mathrm{eff},t})$.
A union bound over $j=1,\dots,M-1$ with $\varepsilon=\sqrt{\log(2M/\eta)/(2n_{\mathrm{eff},t})}$ makes all of them hold with probability at least $1-\eta$.
Because $F$ is continuous, almost surely no $R_s$ equals any $x_j$, so $\hat F(x_j^-)=\hat F(x_j)$. For $x\in[x_j,x_{j+1})$,
monotonicity gives $\hat F(x)-F(x)\le\hat F(x_{j+1}^-)-F(x_j)\le \varepsilon+F(x_{j+1})-F(x_j)=\varepsilon+1/M$ and
$\hat F(x)-F(x)\ge \hat F(x_j)-F(x_{j+1})\ge-\varepsilon-1/M$.
\end{proof}

\section{Proxy streams without model outputs}\label{app:proxy}
Table~\ref{tab:proxy} reports streams built from fields of TruthfulQA, RealToxicityPrompts, SummEval and LongHalQA, as in the earlier version of this work. They contain no model outputs. The features are hand-built summaries of text, the risk is a score derived from annotations in the benchmark, and some orderings use a field linked to the outcome (for example the human consistency rating in SummEval). The two repeated tail segments of the earlier SummEval and LongHalQA streams have been removed. All online methods reach the target coverage to within the standard error. We draw no conclusion from these streams about the value of drift-aware weighting.

\section{Reproduction}\label{app:repro}
All experiments run on a CPU. The model outputs take about three hours (MMLU with a 1.5B-parameter model), forty minutes (TruthfulQA), an hour (RealToxicityPrompts with GPT-2) and about two hours (the Qwen2.5-0.5B replication); the stream analyses take one to two hours each on an eight-core machine. The commands are in the repository \texttt{README}: \texttt{gen\_mmlu\_outputs.py}, \texttt{gen\_truthfulqa\_outputs.py}, \texttt{gen\_rtp\_outputs.py}, then \texttt{run\_phase\_a\_synthetic.py}, \texttt{run\_phase\_a\_baselines.py}, \texttt{run\_real\_streams.py}, \texttt{analyze\_frontier.py}, \texttt{analyze\_real.py}, \texttt{make\_tables.py} and \texttt{make\_figures.py}. The tables and numbers in this paper are produced by \texttt{make\_tables.py} from the result files.

\section{Settings and additional tables}\label{app:hyper}
\begin{table}[H]
\centering
\small
\caption{PromptShift-CRC minus a comparator, paired by seed. Negative means PromptShift-CRC is lower (better for violation; sharper for threshold).}
\label{tab:paired}
\resizebox{\textwidth}{!}{%
\begin{tabular}{lrrrrrr}
\toprule
Stream & Comparator & $\Delta$ post-shift & SE & PS lower in & $\Delta$ threshold & PS lower in \\
\midrule
pilot & matched Adaptive CRC & +0.0016 & 0.0004 & 2/30 & -0.0043 & 30/30 \\
 & no gate & -0.0026 & 0.0004 & 23/30 & +0.0059 & 1/30 \\
 & gate + recency only & +0.0013 & 0.0004 & 4/30 & -0.0038 & 30/30 \\
 & DtACI & +0.0054 & 0.0010 & 4/30 & +0.0121 & 0/30 \\
ablation stream & matched Adaptive CRC & +0.0018 & 0.0005 & 3/30 & -0.0031 & 30/30 \\
 & no gate & -0.0039 & 0.0005 & 26/30 & +0.0026 & 3/30 \\
 & gate + recency only & +0.0013 & 0.0004 & 3/30 & -0.0027 & 30/30 \\
 & DtACI & +0.0079 & 0.0011 & 2/30 & +0.0189 & 0/30 \\
QA topic shift & matched Adaptive CRC & +0.0027 & 0.0004 & 0/30 & -0.0042 & 30/30 \\
 & no gate & -0.0045 & 0.0007 & 24/30 & +0.0091 & 0/30 \\
 & gate + recency only & +0.0028 & 0.0004 & 1/30 & -0.0042 & 30/30 \\
 & DtACI & +0.0144 & 0.0014 & 0/30 & +0.0133 & 0/30 \\
summarization shift & matched Adaptive CRC & +0.0030 & 0.0005 & 2/30 & -0.0042 & 30/30 \\
 & no gate & -0.0052 & 0.0004 & 29/30 & +0.0093 & 0/30 \\
 & gate + recency only & +0.0030 & 0.0004 & 1/30 & -0.0037 & 29/30 \\
 & DtACI & +0.0158 & 0.0011 & 0/30 & +0.0203 & 0/30 \\
safety-intent shift & matched Adaptive CRC & +0.0032 & 0.0004 & 0/30 & -0.0058 & 30/30 \\
 & no gate & -0.0081 & 0.0005 & 30/30 & +0.0162 & 0/30 \\
 & gate + recency only & +0.0032 & 0.0004 & 0/30 & -0.0054 & 30/30 \\
 & DtACI & +0.0133 & 0.0014 & 2/30 & +0.0297 & 0/30 \\
\bottomrule
\end{tabular}
}
\end{table}

\begin{table}[H]
\centering
\small
\caption{MMLU with Qwen2.5-1.5B-Instruct, category blocks (24 streams; mean and sd over streams). Risk: one minus the probability of the correct answer.}
\label{tab:mmlu_catblock}
\resizebox{\textwidth}{!}{%
\begin{tabular}{lrrrrr}
\toprule
Method & Coverage & Post-shift violation & Max local violation & Mean threshold & Intervention rate \\
\midrule
Static & 0.900 $\pm$ 0.022 & 0.097 $\pm$ 0.026 & 0.183 & 0.964 & 1.000 \\
Rolling & 0.903 $\pm$ 0.002 & 0.091 $\pm$ 0.011 & 0.167 & 0.966 & 1.000 \\
Adaptive CRC, window $W$ & 0.900 $\pm$ 0.001 & 0.097 $\pm$ 0.009 & 0.146 & 0.961 & 1.000 \\
Adaptive CRC, matched window & 0.900 $\pm$ 0.001 & 0.099 $\pm$ 0.009 & 0.148 & 0.960 & 1.000 \\
Kernel weights + update & 0.900 $\pm$ 0.001 & 0.097 $\pm$ 0.009 & 0.146 & 0.961 & 1.000 \\
Kernel weights, no update & 0.902 $\pm$ 0.005 & 0.094 $\pm$ 0.011 & 0.177 & 0.965 & 1.000 \\
PromptShift-CRC & 0.900 $\pm$ 0.001 & 0.098 $\pm$ 0.009 & 0.146 & 0.961 & 1.000 \\
Gate + recency only & 0.900 $\pm$ 0.001 & 0.098 $\pm$ 0.009 & 0.146 & 0.961 & 1.000 \\
No gate (fixed mix) & 0.900 $\pm$ 0.001 & 0.098 $\pm$ 0.009 & 0.146 & 0.961 & 1.000 \\
No recency fallback & 0.900 $\pm$ 0.001 & 0.097 $\pm$ 0.009 & 0.146 & 0.961 & 1.000 \\
No online update & 0.903 $\pm$ 0.003 & 0.093 $\pm$ 0.012 & 0.165 & 0.966 & 1.000 \\
DtACI & 0.907 $\pm$ 0.002 & 0.090 $\pm$ 0.010 & 0.150 & 0.966 & 1.000 \\
Density-ratio weights & 0.900 $\pm$ 0.001 & 0.097 $\pm$ 0.008 & 0.144 & 0.961 & 1.000 \\
Quantile tracking & 0.900 $\pm$ 0.001 & 0.098 $\pm$ 0.006 & 0.137 & 0.960 & 1.000 \\
\bottomrule
\end{tabular}
}
\end{table}

\begin{table}[H]
\centering
\small
\caption{MMLU with Qwen2.5-1.5B-Instruct, subject blocks (20 streams; mean and sd over streams). Risk: one minus the probability of the correct answer.}
\label{tab:mmlu_subjectblock}
\resizebox{\textwidth}{!}{%
\begin{tabular}{lrrrrr}
\toprule
Method & Coverage & Post-shift violation & Max local violation & Mean threshold & Intervention rate \\
\midrule
Static & 0.905 $\pm$ 0.029 & 0.094 $\pm$ 0.030 & 0.246 & 0.966 & 1.000 \\
Rolling & 0.901 $\pm$ 0.004 & 0.099 $\pm$ 0.005 & 0.227 & 0.961 & 1.000 \\
Adaptive CRC, window $W$ & 0.900 $\pm$ 0.001 & 0.100 $\pm$ 0.002 & 0.160 & 0.951 & 1.000 \\
Adaptive CRC, matched window & 0.900 $\pm$ 0.001 & 0.100 $\pm$ 0.002 & 0.161 & 0.950 & 1.000 \\
Kernel weights + update & 0.900 $\pm$ 0.001 & 0.100 $\pm$ 0.002 & 0.157 & 0.953 & 1.000 \\
Kernel weights, no update & 0.903 $\pm$ 0.007 & 0.098 $\pm$ 0.008 & 0.235 & 0.964 & 1.000 \\
PromptShift-CRC & 0.900 $\pm$ 0.001 & 0.100 $\pm$ 0.002 & 0.156 & 0.951 & 1.000 \\
Gate + recency only & 0.900 $\pm$ 0.001 & 0.100 $\pm$ 0.002 & 0.156 & 0.952 & 1.000 \\
No gate (fixed mix) & 0.900 $\pm$ 0.001 & 0.100 $\pm$ 0.002 & 0.156 & 0.951 & 1.000 \\
No recency fallback & 0.900 $\pm$ 0.001 & 0.100 $\pm$ 0.002 & 0.157 & 0.953 & 1.000 \\
No online update & 0.903 $\pm$ 0.006 & 0.096 $\pm$ 0.007 & 0.209 & 0.963 & 1.000 \\
DtACI & 0.906 $\pm$ 0.003 & 0.094 $\pm$ 0.003 & 0.157 & 0.956 & 1.000 \\
Density-ratio weights & 0.900 $\pm$ 0.001 & 0.100 $\pm$ 0.002 & 0.157 & 0.952 & 1.000 \\
Quantile tracking & 0.900 $\pm$ 0.001 & 0.100 $\pm$ 0.002 & 0.147 & 0.951 & 1.000 \\
\bottomrule
\end{tabular}
}
\end{table}

\begin{table}[H]
\centering
\small
\caption{MMLU with Qwen2.5-1.5B-Instruct, random order, no shift (20 streams; mean and sd over streams). Risk: one minus the probability of the correct answer.}
\label{tab:mmlu_iid}
\begin{tabular}{lrrrr}
\toprule
Method & Coverage & Max local violation & Mean threshold & Intervention rate \\
\midrule
Static & 0.901 $\pm$ 0.017 & 0.180 & 0.964 & 1.000 \\
Rolling & 0.903 $\pm$ 0.002 & 0.172 & 0.966 & 1.000 \\
Adaptive CRC, window $W$ & 0.900 $\pm$ 0.001 & 0.147 & 0.961 & 1.000 \\
Adaptive CRC, matched window & 0.900 $\pm$ 0.000 & 0.146 & 0.960 & 1.000 \\
Kernel weights + update & 0.900 $\pm$ 0.001 & 0.145 & 0.962 & 1.000 \\
Kernel weights, no update & 0.901 $\pm$ 0.004 & 0.181 & 0.965 & 1.000 \\
PromptShift-CRC & 0.900 $\pm$ 0.000 & 0.145 & 0.961 & 1.000 \\
Gate + recency only & 0.900 $\pm$ 0.000 & 0.145 & 0.961 & 1.000 \\
No gate (fixed mix) & 0.900 $\pm$ 0.000 & 0.144 & 0.961 & 1.000 \\
No recency fallback & 0.900 $\pm$ 0.001 & 0.145 & 0.962 & 1.000 \\
No online update & 0.903 $\pm$ 0.002 & 0.167 & 0.966 & 1.000 \\
DtACI & 0.906 $\pm$ 0.001 & 0.146 & 0.965 & 1.000 \\
Density-ratio weights & 0.900 $\pm$ 0.001 & 0.145 & 0.961 & 1.000 \\
Quantile tracking & 0.900 $\pm$ 0.001 & 0.138 & 0.960 & 1.000 \\
\bottomrule
\end{tabular}
\end{table}

\begin{table}[H]
\centering
\small
\caption{MMLU with Qwen2.5-1.5B-Instruct, embedding-direction order (forward) (1 stream; mean and sd over streams). Risk: one minus the probability of the correct answer.}
\label{tab:mmlu_pc1_forward}
\begin{tabular}{lrrrr}
\toprule
Method & Coverage & Max local violation & Mean threshold & Intervention rate \\
\midrule
Static & 0.902 & 0.200 & 0.965 & 1.000 \\
Rolling & 0.900 & 0.183 & 0.963 & 1.000 \\
Adaptive CRC, window $W$ & 0.900 & 0.158 & 0.960 & 1.000 \\
Adaptive CRC, matched window & 0.900 & 0.167 & 0.959 & 1.000 \\
Kernel weights + update & 0.900 & 0.150 & 0.962 & 1.000 \\
Kernel weights, no update & 0.896 & 0.192 & 0.962 & 1.000 \\
PromptShift-CRC & 0.900 & 0.150 & 0.962 & 1.000 \\
Gate + recency only & 0.900 & 0.150 & 0.961 & 1.000 \\
No gate (fixed mix) & 0.900 & 0.150 & 0.961 & 1.000 \\
No recency fallback & 0.900 & 0.150 & 0.962 & 1.000 \\
No online update & 0.897 & 0.183 & 0.962 & 1.000 \\
DtACI & 0.903 & 0.150 & 0.962 & 1.000 \\
Density-ratio weights & 0.900 & 0.150 & 0.960 & 1.000 \\
Quantile tracking & 0.900 & 0.150 & 0.959 & 1.000 \\
\bottomrule
\end{tabular}
\end{table}

\begin{table}[H]
\centering
\small
\caption{MMLU with Qwen2.5-1.5B-Instruct, embedding-direction order (reversed) (1 stream; mean and sd over streams). Risk: one minus the probability of the correct answer.}
\label{tab:mmlu_pc1_reverse}
\begin{tabular}{lrrrr}
\toprule
Method & Coverage & Max local violation & Mean threshold & Intervention rate \\
\midrule
Static & 0.952 & 0.108 & 0.987 & 1.000 \\
Rolling & 0.902 & 0.175 & 0.963 & 1.000 \\
Adaptive CRC, window $W$ & 0.899 & 0.167 & 0.956 & 1.000 \\
Adaptive CRC, matched window & 0.900 & 0.150 & 0.954 & 1.000 \\
Kernel weights + update & 0.899 & 0.150 & 0.956 & 1.000 \\
Kernel weights, no update & 0.913 & 0.167 & 0.967 & 1.000 \\
PromptShift-CRC & 0.899 & 0.150 & 0.957 & 1.000 \\
Gate + recency only & 0.899 & 0.150 & 0.958 & 1.000 \\
No gate (fixed mix) & 0.899 & 0.150 & 0.957 & 1.000 \\
No recency fallback & 0.899 & 0.150 & 0.956 & 1.000 \\
No online update & 0.906 & 0.158 & 0.965 & 1.000 \\
DtACI & 0.904 & 0.175 & 0.961 & 1.000 \\
Density-ratio weights & 0.899 & 0.150 & 0.957 & 1.000 \\
Quantile tracking & 0.900 & 0.142 & 0.956 & 1.000 \\
\bottomrule
\end{tabular}
\end{table}

\begin{table}[H]
\centering
\small
\caption{MMLU with Qwen2.5-1.5B-Instruct, question-length order (1 stream; mean and sd over streams). Risk: one minus the probability of the correct answer.}
\label{tab:mmlu_length_ascending}
\begin{tabular}{lrrrr}
\toprule
Method & Coverage & Max local violation & Mean threshold & Intervention rate \\
\midrule
Static & 0.921 & 0.158 & 0.976 & 1.000 \\
Rolling & 0.905 & 0.183 & 0.965 & 1.000 \\
Adaptive CRC, window $W$ & 0.900 & 0.158 & 0.958 & 1.000 \\
Adaptive CRC, matched window & 0.899 & 0.150 & 0.955 & 1.000 \\
Kernel weights + update & 0.899 & 0.158 & 0.958 & 1.000 \\
Kernel weights, no update & 0.900 & 0.183 & 0.964 & 1.000 \\
PromptShift-CRC & 0.899 & 0.150 & 0.956 & 1.000 \\
Gate + recency only & 0.899 & 0.150 & 0.957 & 1.000 \\
No gate (fixed mix) & 0.899 & 0.150 & 0.957 & 1.000 \\
No recency fallback & 0.899 & 0.158 & 0.958 & 1.000 \\
No online update & 0.902 & 0.183 & 0.965 & 1.000 \\
DtACI & 0.906 & 0.167 & 0.962 & 1.000 \\
Density-ratio weights & 0.900 & 0.158 & 0.957 & 1.000 \\
Quantile tracking & 0.900 & 0.150 & 0.958 & 1.000 \\
\bottomrule
\end{tabular}
\end{table}

\begin{table}[H]
\centering
\small
\caption{RealToxicityPrompts with GPT-2 continuations, prompt-toxicity bands in blocks (24 streams; mean and sd over streams). Risk: toxicity probability of the continuation.}
\label{tab:rtp_bandblock}
\resizebox{\textwidth}{!}{%
\begin{tabular}{lrrrrr}
\toprule
Method & Coverage & Post-shift violation & Max local violation & Mean threshold & Intervention rate \\
\midrule
Static & 0.867 $\pm$ 0.081 & 0.144 $\pm$ 0.102 & 0.390 & 0.184 & 0.250 \\
Rolling & 0.901 $\pm$ 0.005 & 0.114 $\pm$ 0.028 & 0.297 & 0.172 & 0.204 \\
Adaptive CRC, window $W$ & 0.900 $\pm$ 0.000 & 0.095 $\pm$ 0.011 & 0.208 & 0.184 & 0.180 \\
Adaptive CRC, matched window & 0.900 $\pm$ 0.000 & 0.100 $\pm$ 0.008 & 0.207 & 0.186 & 0.174 \\
Kernel weights + update & 0.900 $\pm$ 0.000 & 0.093 $\pm$ 0.012 & 0.206 & 0.183 & 0.179 \\
Kernel weights, no update & 0.898 $\pm$ 0.012 & 0.131 $\pm$ 0.037 & 0.348 & 0.155 & 0.155 \\
PromptShift-CRC & 0.900 $\pm$ 0.000 & 0.095 $\pm$ 0.010 & 0.204 & 0.185 & 0.179 \\
Gate + recency only & 0.900 $\pm$ 0.000 & 0.095 $\pm$ 0.010 & 0.204 & 0.185 & 0.178 \\
No gate (fixed mix) & 0.900 $\pm$ 0.000 & 0.096 $\pm$ 0.009 & 0.206 & 0.183 & 0.175 \\
No recency fallback & 0.900 $\pm$ 0.000 & 0.093 $\pm$ 0.012 & 0.206 & 0.183 & 0.179 \\
No online update & 0.904 $\pm$ 0.008 & 0.112 $\pm$ 0.027 & 0.290 & 0.161 & 0.164 \\
DtACI & 0.906 $\pm$ 0.001 & 0.097 $\pm$ 0.011 & 0.217 & 0.194 & 0.199 \\
Density-ratio weights & 0.900 $\pm$ 0.000 & 0.095 $\pm$ 0.010 & 0.204 & 0.183 & 0.177 \\
Quantile tracking & 0.900 $\pm$ 0.005 & 0.099 $\pm$ 0.017 & 0.244 & 0.167 & 0.151 \\
\bottomrule
\end{tabular}
}
\end{table}

\begin{table}[H]
\centering
\small
\caption{RealToxicityPrompts with GPT-2 continuations, random order, no shift (20 streams; mean and sd over streams). Risk: toxicity probability of the continuation.}
\label{tab:rtp_iid}
\begin{tabular}{lrrrr}
\toprule
Method & Coverage & Max local violation & Mean threshold & Intervention rate \\
\midrule
Static & 0.906 $\pm$ 0.023 & 0.211 & 0.096 & 0.000 \\
Rolling & 0.902 $\pm$ 0.002 & 0.224 & 0.080 & 0.000 \\
Adaptive CRC, window $W$ & 0.900 $\pm$ 0.000 & 0.192 & 0.118 & 0.032 \\
Adaptive CRC, matched window & 0.900 $\pm$ 0.000 & 0.193 & 0.136 & 0.052 \\
Kernel weights + update & 0.900 $\pm$ 0.000 & 0.192 & 0.102 & 0.016 \\
Kernel weights, no update & 0.903 $\pm$ 0.003 & 0.218 & 0.072 & 0.000 \\
PromptShift-CRC & 0.900 $\pm$ 0.000 & 0.189 & 0.114 & 0.028 \\
Gate + recency only & 0.900 $\pm$ 0.000 & 0.188 & 0.116 & 0.029 \\
No gate (fixed mix) & 0.900 $\pm$ 0.000 & 0.189 & 0.117 & 0.031 \\
No recency fallback & 0.900 $\pm$ 0.000 & 0.192 & 0.102 & 0.016 \\
No online update & 0.904 $\pm$ 0.002 & 0.210 & 0.079 & 0.000 \\
DtACI & 0.907 $\pm$ 0.001 & 0.196 & 0.124 & 0.019 \\
Density-ratio weights & 0.900 $\pm$ 0.000 & 0.187 & 0.111 & 0.024 \\
Quantile tracking & 0.900 $\pm$ 0.001 & 0.200 & 0.079 & 0.000 \\
\bottomrule
\end{tabular}
\end{table}

\begin{table}[H]
\centering
\small
\caption{RealToxicityPrompts with GPT-2 continuations, embedding-direction order (forward) (1 stream; mean and sd over streams). Risk: toxicity probability of the continuation.}
\label{tab:rtp_pc1_forward}
\begin{tabular}{lrrrr}
\toprule
Method & Coverage & Max local violation & Mean threshold & Intervention rate \\
\midrule
Static & 0.684 & 0.717 & 0.002 & 0.000 \\
Rolling & 0.894 & 0.250 & 0.091 & 0.000 \\
Adaptive CRC, window $W$ & 0.899 & 0.200 & 0.140 & 0.052 \\
Adaptive CRC, matched window & 0.899 & 0.200 & 0.174 & 0.092 \\
Kernel weights + update & 0.899 & 0.200 & 0.135 & 0.042 \\
Kernel weights, no update & 0.885 & 0.267 & 0.067 & 0.000 \\
PromptShift-CRC & 0.899 & 0.200 & 0.138 & 0.049 \\
Gate + recency only & 0.899 & 0.200 & 0.142 & 0.055 \\
No gate (fixed mix) & 0.899 & 0.200 & 0.145 & 0.053 \\
No recency fallback & 0.899 & 0.200 & 0.135 & 0.042 \\
No online update & 0.891 & 0.250 & 0.080 & 0.000 \\
DtACI & 0.905 & 0.217 & 0.164 & 0.048 \\
Density-ratio weights & 0.900 & 0.183 & 0.144 & 0.057 \\
Quantile tracking & 0.896 & 0.217 & 0.089 & 0.000 \\
\bottomrule
\end{tabular}
\end{table}

\begin{table}[H]
\centering
\small
\caption{RealToxicityPrompts with GPT-2 continuations, embedding-direction order (reversed) (1 stream; mean and sd over streams). Risk: toxicity probability of the continuation.}
\label{tab:rtp_pc1_reverse}
\begin{tabular}{lrrrr}
\toprule
Method & Coverage & Max local violation & Mean threshold & Intervention rate \\
\midrule
Static & 0.947 & 0.167 & 0.283 & 0.000 \\
Rolling & 0.909 & 0.250 & 0.091 & 0.000 \\
Adaptive CRC, window $W$ & 0.901 & 0.183 & 0.109 & 0.035 \\
Adaptive CRC, matched window & 0.901 & 0.183 & 0.121 & 0.049 \\
Kernel weights + update & 0.901 & 0.183 & 0.090 & 0.015 \\
Kernel weights, no update & 0.917 & 0.217 & 0.099 & 0.000 \\
PromptShift-CRC & 0.901 & 0.183 & 0.110 & 0.034 \\
Gate + recency only & 0.901 & 0.183 & 0.109 & 0.033 \\
No gate (fixed mix) & 0.901 & 0.183 & 0.111 & 0.037 \\
No recency fallback & 0.901 & 0.183 & 0.090 & 0.015 \\
No online update & 0.914 & 0.217 & 0.097 & 0.000 \\
DtACI & 0.908 & 0.183 & 0.124 & 0.023 \\
Density-ratio weights & 0.901 & 0.183 & 0.098 & 0.026 \\
Quantile tracking & 0.903 & 0.217 & 0.086 & 0.000 \\
\bottomrule
\end{tabular}
\end{table}

\begin{table}[H]
\centering
\small
\caption{RealToxicityPrompts with GPT-2 continuations, prompt-length order (1 stream; mean and sd over streams). Risk: toxicity probability of the continuation.}
\label{tab:rtp_length_ascending}
\begin{tabular}{lrrrr}
\toprule
Method & Coverage & Max local violation & Mean threshold & Intervention rate \\
\midrule
Static & 0.940 & 0.167 & 0.234 & 0.000 \\
Rolling & 0.905 & 0.200 & 0.089 & 0.000 \\
Adaptive CRC, window $W$ & 0.900 & 0.183 & 0.122 & 0.045 \\
Adaptive CRC, matched window & 0.900 & 0.200 & 0.142 & 0.067 \\
Kernel weights + update & 0.900 & 0.200 & 0.111 & 0.025 \\
Kernel weights, no update & 0.906 & 0.217 & 0.076 & 0.000 \\
PromptShift-CRC & 0.900 & 0.183 & 0.120 & 0.035 \\
Gate + recency only & 0.900 & 0.183 & 0.128 & 0.044 \\
No gate (fixed mix) & 0.900 & 0.183 & 0.125 & 0.038 \\
No recency fallback & 0.900 & 0.200 & 0.111 & 0.025 \\
No online update & 0.906 & 0.217 & 0.083 & 0.000 \\
DtACI & 0.906 & 0.183 & 0.126 & 0.023 \\
Density-ratio weights & 0.900 & 0.200 & 0.117 & 0.032 \\
Quantile tracking & 0.901 & 0.200 & 0.082 & 0.000 \\
\bottomrule
\end{tabular}
\end{table}

\begin{table}[H]
\centering
\small
\caption{RealToxicityPrompts with GPT-2 continuations, ascending prompt toxicity (1 stream; mean and sd over streams). Risk: toxicity probability of the continuation.}
\label{tab:rtp_prompt_toxicity_ascending}
\begin{tabular}{lrrrr}
\toprule
Method & Coverage & Max local violation & Mean threshold & Intervention rate \\
\midrule
Static & 0.695 & 0.600 & 0.002 & 0.000 \\
Rolling & 0.888 & 0.233 & 0.141 & 0.116 \\
Adaptive CRC, window $W$ & 0.899 & 0.183 & 0.175 & 0.155 \\
Adaptive CRC, matched window & 0.899 & 0.183 & 0.180 & 0.150 \\
Kernel weights + update & 0.899 & 0.183 & 0.175 & 0.152 \\
Kernel weights, no update & 0.868 & 0.283 & 0.089 & 0.049 \\
PromptShift-CRC & 0.899 & 0.167 & 0.175 & 0.158 \\
Gate + recency only & 0.899 & 0.167 & 0.176 & 0.156 \\
No gate (fixed mix) & 0.899 & 0.183 & 0.172 & 0.150 \\
No recency fallback & 0.899 & 0.183 & 0.175 & 0.152 \\
No online update & 0.881 & 0.250 & 0.123 & 0.095 \\
DtACI & 0.904 & 0.200 & 0.176 & 0.141 \\
Density-ratio weights & 0.899 & 0.183 & 0.178 & 0.154 \\
Quantile tracking & 0.888 & 0.250 & 0.129 & 0.087 \\
\bottomrule
\end{tabular}
\end{table}

\begin{table}[H]
\centering
\small
\caption{RealToxicityPrompts with GPT-2 continuations, descending prompt toxicity (1 stream; mean and sd over streams). Risk: toxicity probability of the continuation.}
\label{tab:rtp_prompt_toxicity_descending}
\begin{tabular}{lrrrr}
\toprule
Method & Coverage & Max local violation & Mean threshold & Intervention rate \\
\midrule
Static & 0.997 & 0.067 & 0.979 & 1.000 \\
Rolling & 0.913 & 0.217 & 0.140 & 0.117 \\
Adaptive CRC, window $W$ & 0.900 & 0.183 & 0.120 & 0.075 \\
Adaptive CRC, matched window & 0.900 & 0.183 & 0.127 & 0.085 \\
Kernel weights + update & 0.900 & 0.183 & 0.117 & 0.084 \\
Kernel weights, no update & 0.935 & 0.183 & 0.203 & 0.187 \\
PromptShift-CRC & 0.900 & 0.183 & 0.120 & 0.077 \\
Gate + recency only & 0.900 & 0.183 & 0.119 & 0.076 \\
No gate (fixed mix) & 0.900 & 0.183 & 0.118 & 0.074 \\
No recency fallback & 0.900 & 0.183 & 0.117 & 0.084 \\
No online update & 0.923 & 0.200 & 0.168 & 0.140 \\
DtACI & 0.907 & 0.183 & 0.124 & 0.085 \\
Density-ratio weights & 0.900 & 0.183 & 0.119 & 0.082 \\
Quantile tracking & 0.912 & 0.183 & 0.163 & 0.125 \\
\bottomrule
\end{tabular}
\end{table}

\begin{table}[H]
\centering
\small
\caption{RealToxicityPrompts with Qwen2.5-0.5B continuations (replication), prompt-toxicity bands in blocks (24 streams; mean and sd over streams). Risk: toxicity probability of the continuation.}
\label{tab:rtp_bandblock_qwen}
\resizebox{\textwidth}{!}{%
\begin{tabular}{lrrrrr}
\toprule
Method & Coverage & Post-shift violation & Max local violation & Mean threshold & Intervention rate \\
\midrule
Static & 0.863 $\pm$ 0.093 & 0.144 $\pm$ 0.107 & 0.369 & 0.160 & 0.083 \\
Rolling & 0.901 $\pm$ 0.005 & 0.112 $\pm$ 0.033 & 0.284 & 0.147 & 0.074 \\
Adaptive CRC, window $W$ & 0.900 $\pm$ 0.000 & 0.092 $\pm$ 0.011 & 0.208 & 0.167 & 0.116 \\
Adaptive CRC, matched window & 0.900 $\pm$ 0.000 & 0.097 $\pm$ 0.008 & 0.209 & 0.173 & 0.126 \\
Kernel weights + update & 0.900 $\pm$ 0.000 & 0.090 $\pm$ 0.012 & 0.209 & 0.162 & 0.111 \\
Kernel weights, no update & 0.898 $\pm$ 0.014 & 0.130 $\pm$ 0.043 & 0.342 & 0.128 & 0.012 \\
PromptShift-CRC & 0.900 $\pm$ 0.000 & 0.092 $\pm$ 0.011 & 0.203 & 0.168 & 0.118 \\
Gate + recency only & 0.900 $\pm$ 0.000 & 0.092 $\pm$ 0.011 & 0.203 & 0.168 & 0.116 \\
No gate (fixed mix) & 0.900 $\pm$ 0.000 & 0.092 $\pm$ 0.010 & 0.202 & 0.167 & 0.118 \\
No recency fallback & 0.900 $\pm$ 0.000 & 0.090 $\pm$ 0.012 & 0.209 & 0.162 & 0.111 \\
No online update & 0.904 $\pm$ 0.009 & 0.110 $\pm$ 0.032 & 0.281 & 0.135 & 0.043 \\
DtACI & 0.906 $\pm$ 0.001 & 0.094 $\pm$ 0.011 & 0.220 & 0.178 & 0.133 \\
Density-ratio weights & 0.900 $\pm$ 0.000 & 0.093 $\pm$ 0.010 & 0.201 & 0.166 & 0.115 \\
Quantile tracking & 0.900 $\pm$ 0.004 & 0.100 $\pm$ 0.018 & 0.229 & 0.142 & 0.023 \\
\bottomrule
\end{tabular}
}
\end{table}

\begin{table}[H]
\centering
\small
\caption{TruthfulQA with Qwen2.5-1.5B-Instruct, category blocks (20 streams; mean and sd over streams). Risk: one minus the probability of the correct answer.}
\label{tab:truthfulqa_categoryblock}
\resizebox{\textwidth}{!}{%
\begin{tabular}{lrrrrr}
\toprule
Method & Coverage & Post-shift violation & Max local violation & Mean threshold & Intervention rate \\
\midrule
Static & 0.883 $\pm$ 0.059 & 0.122 $\pm$ 0.063 & 0.250 & 1.000 & 1.000 \\
Rolling & 0.895 $\pm$ 0.007 & 0.105 $\pm$ 0.011 & 0.223 & 1.000 & 1.000 \\
Adaptive CRC, window $W$ & 0.900 $\pm$ 0.002 & 0.099 $\pm$ 0.008 & 0.182 & 1.000 & 1.000 \\
Adaptive CRC, matched window & 0.900 $\pm$ 0.002 & 0.099 $\pm$ 0.006 & 0.165 & 1.000 & 1.000 \\
Kernel weights + update & 0.900 $\pm$ 0.002 & 0.101 $\pm$ 0.008 & 0.182 & 1.000 & 1.000 \\
Kernel weights, no update & 0.901 $\pm$ 0.012 & 0.101 $\pm$ 0.013 & 0.229 & 1.000 & 1.000 \\
PromptShift-CRC & 0.900 $\pm$ 0.002 & 0.100 $\pm$ 0.007 & 0.168 & 1.000 & 1.000 \\
Gate + recency only & 0.900 $\pm$ 0.002 & 0.100 $\pm$ 0.007 & 0.169 & 1.000 & 1.000 \\
No gate (fixed mix) & 0.901 $\pm$ 0.002 & 0.100 $\pm$ 0.007 & 0.167 & 1.000 & 1.000 \\
No recency fallback & 0.900 $\pm$ 0.002 & 0.101 $\pm$ 0.008 & 0.182 & 1.000 & 1.000 \\
No online update & 0.906 $\pm$ 0.010 & 0.096 $\pm$ 0.012 & 0.197 & 1.000 & 1.000 \\
DtACI & 0.915 $\pm$ 0.004 & 0.085 $\pm$ 0.008 & 0.156 & 1.000 & 1.000 \\
Density-ratio weights & 0.901 $\pm$ 0.002 & 0.100 $\pm$ 0.007 & 0.174 & 1.000 & 1.000 \\
Quantile tracking & 0.900 $\pm$ 0.000 & 0.100 $\pm$ 0.002 & 0.121 & 0.998 & 1.000 \\
\bottomrule
\end{tabular}
}
\end{table}

\begin{table}[H]
\centering
\small
\caption{TruthfulQA with Qwen2.5-1.5B-Instruct, random order, no shift (20 streams; mean and sd over streams). Risk: one minus the probability of the correct answer.}
\label{tab:truthfulqa_iid}
\begin{tabular}{lrrrr}
\toprule
Method & Coverage & Max local violation & Mean threshold & Intervention rate \\
\midrule
Static & 0.899 $\pm$ 0.043 & 0.186 & 1.000 & 1.000 \\
Rolling & 0.902 $\pm$ 0.004 & 0.190 & 1.000 & 1.000 \\
Adaptive CRC, window $W$ & 0.900 $\pm$ 0.001 & 0.170 & 1.000 & 1.000 \\
Adaptive CRC, matched window & 0.900 $\pm$ 0.002 & 0.165 & 1.000 & 1.000 \\
Kernel weights + update & 0.900 $\pm$ 0.001 & 0.168 & 1.000 & 1.000 \\
Kernel weights, no update & 0.906 $\pm$ 0.006 & 0.176 & 1.000 & 1.000 \\
PromptShift-CRC & 0.900 $\pm$ 0.002 & 0.163 & 1.000 & 1.000 \\
Gate + recency only & 0.900 $\pm$ 0.002 & 0.164 & 1.000 & 1.000 \\
No gate (fixed mix) & 0.900 $\pm$ 0.002 & 0.163 & 1.000 & 1.000 \\
No recency fallback & 0.900 $\pm$ 0.001 & 0.168 & 1.000 & 1.000 \\
No online update & 0.908 $\pm$ 0.004 & 0.162 & 1.000 & 1.000 \\
DtACI & 0.915 $\pm$ 0.006 & 0.150 & 1.000 & 1.000 \\
Density-ratio weights & 0.900 $\pm$ 0.001 & 0.167 & 1.000 & 1.000 \\
Quantile tracking & 0.900 $\pm$ 0.000 & 0.117 & 0.999 & 1.000 \\
\bottomrule
\end{tabular}
\end{table}

\begin{table}[H]
\centering
\small
\caption{TruthfulQA with Qwen2.5-1.5B-Instruct, embedding-direction order (forward) (1 stream; mean and sd over streams). Risk: one minus the probability of the correct answer.}
\label{tab:truthfulqa_pc1_forward}
\begin{tabular}{lrrrr}
\toprule
Method & Coverage & Max local violation & Mean threshold & Intervention rate \\
\midrule
Static & 0.897 & 0.217 & 1.000 & 1.000 \\
Rolling & 0.905 & 0.200 & 1.000 & 1.000 \\
Adaptive CRC, window $W$ & 0.904 & 0.150 & 1.000 & 1.000 \\
Adaptive CRC, matched window & 0.904 & 0.167 & 1.000 & 1.000 \\
Kernel weights + update & 0.904 & 0.150 & 1.000 & 1.000 \\
Kernel weights, no update & 0.904 & 0.217 & 1.000 & 1.000 \\
PromptShift-CRC & 0.904 & 0.167 & 1.000 & 1.000 \\
Gate + recency only & 0.904 & 0.150 & 1.000 & 1.000 \\
No gate (fixed mix) & 0.904 & 0.167 & 1.000 & 1.000 \\
No recency fallback & 0.904 & 0.150 & 1.000 & 1.000 \\
No online update & 0.908 & 0.200 & 1.000 & 1.000 \\
DtACI & 0.916 & 0.150 & 1.000 & 1.000 \\
Density-ratio weights & 0.904 & 0.150 & 1.000 & 1.000 \\
Quantile tracking & 0.900 & 0.117 & 0.999 & 1.000 \\
\bottomrule
\end{tabular}
\end{table}

\begin{table}[H]
\centering
\small
\caption{TruthfulQA with Qwen2.5-1.5B-Instruct, embedding-direction order (reversed) (1 stream; mean and sd over streams). Risk: one minus the probability of the correct answer.}
\label{tab:truthfulqa_pc1_reverse}
\begin{tabular}{lrrrr}
\toprule
Method & Coverage & Max local violation & Mean threshold & Intervention rate \\
\midrule
Static & 0.858 & 0.283 & 1.000 & 1.000 \\
Rolling & 0.904 & 0.183 & 1.000 & 1.000 \\
Adaptive CRC, window $W$ & 0.898 & 0.167 & 1.000 & 1.000 \\
Adaptive CRC, matched window & 0.900 & 0.167 & 1.000 & 1.000 \\
Kernel weights + update & 0.900 & 0.167 & 1.000 & 1.000 \\
Kernel weights, no update & 0.901 & 0.200 & 1.000 & 1.000 \\
PromptShift-CRC & 0.900 & 0.167 & 1.000 & 1.000 \\
Gate + recency only & 0.900 & 0.150 & 1.000 & 1.000 \\
No gate (fixed mix) & 0.900 & 0.150 & 1.000 & 1.000 \\
No recency fallback & 0.900 & 0.167 & 1.000 & 1.000 \\
No online update & 0.911 & 0.183 & 1.000 & 1.000 \\
DtACI & 0.914 & 0.150 & 1.000 & 1.000 \\
Density-ratio weights & 0.901 & 0.150 & 1.000 & 1.000 \\
Quantile tracking & 0.900 & 0.117 & 0.999 & 1.000 \\
\bottomrule
\end{tabular}
\end{table}

\begin{table}[H]
\centering
\small
\caption{TruthfulQA with Qwen2.5-1.5B-Instruct, question-length order (1 stream; mean and sd over streams). Risk: one minus the probability of the correct answer.}
\label{tab:truthfulqa_length_ascending}
\begin{tabular}{lrrrr}
\toprule
Method & Coverage & Max local violation & Mean threshold & Intervention rate \\
\midrule
Static & 0.876 & 0.217 & 1.000 & 1.000 \\
Rolling & 0.894 & 0.217 & 1.000 & 1.000 \\
Adaptive CRC, window $W$ & 0.898 & 0.167 & 1.000 & 1.000 \\
Adaptive CRC, matched window & 0.900 & 0.183 & 1.000 & 1.000 \\
Kernel weights + update & 0.900 & 0.150 & 1.000 & 1.000 \\
Kernel weights, no update & 0.900 & 0.167 & 1.000 & 1.000 \\
PromptShift-CRC & 0.900 & 0.150 & 1.000 & 1.000 \\
Gate + recency only & 0.900 & 0.150 & 1.000 & 1.000 \\
No gate (fixed mix) & 0.900 & 0.150 & 1.000 & 1.000 \\
No recency fallback & 0.900 & 0.150 & 1.000 & 1.000 \\
No online update & 0.905 & 0.167 & 1.000 & 1.000 \\
DtACI & 0.915 & 0.150 & 1.000 & 1.000 \\
Density-ratio weights & 0.900 & 0.167 & 1.000 & 1.000 \\
Quantile tracking & 0.900 & 0.117 & 0.999 & 1.000 \\
\bottomrule
\end{tabular}
\end{table}

\begin{table}[H]
\centering
\small
\caption{Frontier gap on TRUTHFULQA streams (post-shift violation minus the Adaptive-CRC window-sweep curve at the same mean threshold; mean over orderings with standard error). Negative is better than Adaptive CRC. Where a method's mean threshold lies outside the range of the window sweep the curve is clipped to its nearest end; the last column repeats the PromptShift-CRC gap over the orderings whose threshold lies inside the range.}
\label{tab:frontier_truthfulqa}
\resizebox{\textwidth}{!}{%
\begin{tabular}{lrrrrrr}
\toprule
Ordering & PromptShift-CRC & Gate + recency only & Kernel + update & DtACI & Density ratio & PromptShift-CRC, inside range \\
\midrule
category blocks & -0.0004 (0.0007) & -0.0004 (0.0007) & +0.0003 (0.0004) & -0.0162 (0.0012) & -0.0009 (0.0006) & -0.0004 (0.0007), 17 of 20 \\
\bottomrule
\end{tabular}
}
\end{table}

\begin{table}[H]
\centering
\small
\caption{Proxy streams built from benchmark fields (no model outputs). Coverage with moving-block-bootstrap standard error.}
\label{tab:proxy}
\resizebox{\textwidth}{!}{%
\begin{tabular}{lrrrrr}
\toprule
Stream & Method & Coverage & SE & Mean threshold & Max local violation \\
\midrule
TruthfulQA & Static & 0.994 & 0.003 & 0.568 & 0.033 \\
 & Rolling & 0.916 & 0.027 & 0.401 & 0.242 \\
 & Adaptive CRC, window $W$ & 0.909 & 0.025 & 0.394 & 0.217 \\
 & Adaptive CRC, matched window & 0.904 & 0.014 & 0.381 & 0.175 \\
 & PromptShift-CRC & 0.907 & 0.018 & 0.387 & 0.175 \\
 & Gate + recency only & 0.907 & 0.017 & 0.385 & 0.167 \\
 & No gate (fixed mix) & 0.907 & 0.018 & 0.388 & 0.175 \\
SummEval & Static & 0.888 & 0.007 & 0.263 & 0.183 \\
 & Rolling & 0.900 & 0.008 & 0.420 & 0.158 \\
 & Adaptive CRC, window $W$ & 0.902 & 0.006 & 0.488 & 0.142 \\
 & Adaptive CRC, matched window & 0.902 & 0.006 & 0.523 & 0.150 \\
 & PromptShift-CRC & 0.902 & 0.006 & 0.485 & 0.150 \\
 & Gate + recency only & 0.902 & 0.006 & 0.524 & 0.150 \\
 & No gate (fixed mix) & 0.902 & 0.006 & 0.487 & 0.150 \\
LongHalQA & Static & 0.205 & 0.115 & 0.286 & 1.000 \\
 & Rolling & 0.725 & 0.057 & 0.890 & 0.733 \\
 & Adaptive CRC, window $W$ & 0.897 & 0.017 & 0.880 & 0.258 \\
 & Adaptive CRC, matched window & 0.897 & 0.009 & 0.841 & 0.175 \\
 & PromptShift-CRC & 0.897 & 0.020 & 0.921 & 0.258 \\
 & Gate + recency only & 0.897 & 0.013 & 0.854 & 0.217 \\
 & No gate (fixed mix) & 0.897 & 0.019 & 0.901 & 0.250 \\
RealToxicityPrompts & Static & 0.891 & 0.005 & 0.179 & 0.208 \\
 & Rolling & 0.896 & 0.004 & 0.183 & 0.175 \\
 & Adaptive CRC, window $W$ & 0.900 & 0.003 & 0.187 & 0.167 \\
 & Adaptive CRC, matched window & 0.900 & 0.002 & 0.188 & 0.167 \\
 & PromptShift-CRC & 0.900 & 0.003 & 0.185 & 0.158 \\
 & Gate + recency only & 0.900 & 0.002 & 0.187 & 0.167 \\
 & No gate (fixed mix) & 0.900 & 0.003 & 0.186 & 0.158 \\
\bottomrule
\end{tabular}
}
\end{table}

\begin{table}[H]
\centering
\small
\caption{RealToxicityPrompts across generators, classifiers and target levels (24 band orderings each). PromptShift-CRC minus the matched baseline, post-shift violation with standard error and number of orderings in which PromptShift-CRC is lower, and mean threshold; frontier gaps of PromptShift-CRC and DtACI.}
\label{tab:rtp_robust}
\resizebox{\textwidth}{!}{%
\begin{tabular}{lrrrrrr}
\toprule
Configuration & PS $-$ matched post-shift & PS lower in & PS $-$ matched threshold & PS lower in & Frontier gap PS & Frontier gap DtACI \\
\midrule
GPT-2, toxic-bert, $\alpha=0.1$ (main) & -0.0052 (0.0011) & 19/24 & -0.0011 & 15/24 & -0.0010 (0.0010) & -0.0030 (0.0015) \\
Qwen2.5-0.5B, unbiased-toxic-roberta, $\alpha=0.1$ (replication) & -0.0050 (0.0015) & 16/24 & -0.0053 & 20/24 & -0.0003 (0.0007) & -0.0019 (0.0015) \\
GPT-2, toxic-bert, $\alpha=0.05$ & -0.0039 (0.0009) & 19/24 & -0.0166 & 23/24 & -0.0010 (0.0009) & -0.0037 (0.0012) \\
GPT-2, toxic-bert, $\alpha=0.2$ & -0.0031 (0.0019) & 13/24 & -0.0025 & 21/24 & -0.0017 (0.0015) & +0.0027 (0.0019) \\
\bottomrule
\end{tabular}
}
\end{table}

\begin{table}[H]
\centering
\small
\caption{MMLU category blocks (24 orderings) at three target levels, all constants as declared. The spread is the range over the four categories of the $(1-\alpha)$ quantile of the risk. PromptShift-CRC minus the matched baseline in post-shift violation (standard error, number of orderings in which PromptShift-CRC is lower) and in mean threshold; frontier gaps of PromptShift-CRC (over all orderings and over those inside the range of the window sweep), kernel weights with the update, and DtACI.}
\label{tab:mmlu_robust}
\resizebox{\textwidth}{!}{%
\begin{tabular}{lrrrrrrrr}
\toprule
Level & Spread of target quantile & PS $-$ matched post-shift & PS lower in & PS $-$ matched threshold & Frontier gap PS & PS, inside range & Kernel + update & DtACI \\
\midrule
$\alpha=0.1$ (main) & 0.025 & -0.0010 (0.0007) & 9/24 & +0.0010 & -0.0007 (0.0006) & +0.0001 (0.0006), 17/24 & -0.0007 (0.0006) & -0.0084 (0.0012) \\
$\alpha=0.2$ & 0.077 & -0.0016 (0.0009) & 12/24 & +0.0033 & +0.0003 (0.0007) & +0.0007 (0.0007), 20/24 & +0.0006 (0.0008) & -0.0031 (0.0013) \\
$\alpha=0.3$ & 0.242 & -0.0004 (0.0013) & 11/24 & -0.0000 & +0.0002 (0.0012) & +0.0004 (0.0013), 23/24 & +0.0009 (0.0019) & -0.0002 (0.0015) \\
\bottomrule
\end{tabular}
}
\end{table}

\begin{table}[H]
\centering
\small
\caption{All four L-ARC configurations (the first of each group is the pre-declared primary one). Coverage, frontier gap and the difference from PromptShift-CRC, as in the previous table.}
\label{tab:larc_all}
\resizebox{\textwidth}{!}{%
\begin{tabular}{lrrrrr}
\toprule
Setting & Configuration & Coverage & Frontier gap & $-$ PS post-shift & $-$ PS threshold \\
\midrule
synthetic (mean of the five streams) & L-ARC, authors' defaults ($\eta_1=1$, length scale $1$) & 0.898 & +0.0277 (0.0005) & +0.0237 (0.0004) & -0.0069 \\
 & L-ARC, $\eta_1=1$, length scale matched to $h$ & 0.898 & +0.0284 (0.0005) & +0.0205 (0.0004) & -0.0179 \\
 & L-ARC, $\eta_1=0.1$, length scale $1$ & 0.858 & +0.1970 (0.0009) & +0.1925 (0.0009) & -0.0074 \\
 & L-ARC, $\eta_1=0.1$, length scale matched to $h$ & 0.860 & +0.1936 (0.0009) & +0.1885 (0.0009) & -0.0087 \\
MMLU, category blocks & L-ARC, authors' defaults ($\eta_1=1$, length scale $1$) & 0.900 & +0.0001 (0.0010) & +0.0019 (0.0011) & -0.0091 \\
 & L-ARC, $\eta_1=1$, length scale matched to $h$ & 0.900 & -0.0004 (0.0010) & +0.0015 (0.0011) & -0.0103 \\
 & L-ARC, $\eta_1=0.1$, length scale $1$ & 0.900 & -0.0007 (0.0012) & -0.0002 (0.0012) & +0.0018 \\
 & L-ARC, $\eta_1=0.1$, length scale matched to $h$ & 0.900 & -0.0007 (0.0013) & -0.0002 (0.0013) & +0.0016 \\
MMLU, subject blocks & L-ARC, authors' defaults ($\eta_1=1$, length scale $1$) & 0.900 & -0.0002 (0.0004) & +0.0002 (0.0002) & -0.0087 \\
 & L-ARC, $\eta_1=1$, length scale matched to $h$ & 0.900 & -0.0003 (0.0004) & +0.0002 (0.0002) & -0.0099 \\
 & L-ARC, $\eta_1=0.1$, length scale $1$ & 0.900 & +0.0002 (0.0008) & +0.0003 (0.0008) & +0.0077 \\
 & L-ARC, $\eta_1=0.1$, length scale matched to $h$ & 0.900 & +0.0002 (0.0007) & +0.0003 (0.0007) & +0.0072 \\
RealToxicityPrompts, GPT-2 & L-ARC, authors' defaults ($\eta_1=1$, length scale $1$) & 0.900 & -0.0033 (0.0033) & +0.0011 (0.0027) & -0.0016 \\
 & L-ARC, $\eta_1=1$, length scale matched to $h$ & 0.900 & -0.0036 (0.0032) & +0.0013 (0.0027) & +0.0060 \\
 & L-ARC, $\eta_1=0.1$, length scale $1$ & 0.909 & -0.0006 (0.0094) & +0.0016 (0.0084) & -0.0359 \\
 & L-ARC, $\eta_1=0.1$, length scale matched to $h$ & 0.907 & +0.0005 (0.0083) & +0.0026 (0.0073) & -0.0377 \\
RealToxicityPrompts, Qwen2.5-0.5B & L-ARC, authors' defaults ($\eta_1=1$, length scale $1$) & 0.900 & +0.0064 (0.0028) & +0.0064 (0.0027) & -0.0107 \\
 & L-ARC, $\eta_1=1$, length scale matched to $h$ & 0.900 & +0.0038 (0.0030) & +0.0069 (0.0025) & -0.0050 \\
 & L-ARC, $\eta_1=0.1$, length scale $1$ & 0.907 & +0.0044 (0.0083) & +0.0070 (0.0073) & -0.0431 \\
 & L-ARC, $\eta_1=0.1$, length scale matched to $h$ & 0.905 & +0.0048 (0.0076) & +0.0074 (0.0067) & -0.0444 \\
TruthfulQA, category blocks & L-ARC, authors' defaults ($\eta_1=1$, length scale $1$) & 0.901 & -0.0008 (0.0011) & -0.0017 (0.0013) & -0.0044 \\
 & L-ARC, $\eta_1=1$, length scale matched to $h$ & 0.901 & -0.0007 (0.0009) & -0.0015 (0.0014) & -0.0049 \\
 & L-ARC, $\eta_1=0.1$, length scale $1$ & 0.901 & -0.0000 (0.0013) & -0.0009 (0.0014) & -0.0006 \\
 & L-ARC, $\eta_1=0.1$, length scale matched to $h$ & 0.901 & -0.0003 (0.0012) & -0.0012 (0.0015) & -0.0007 \\
\bottomrule
\end{tabular}
}
\end{table}

\begin{table}[H]
\centering
\small
\caption{Frontier gap (post-shift violation minus the Adaptive CRC window-sweep curve at the method's own mean threshold; mean over 24 orderings with standard error) for kernel bandwidths $h/4$ to $2h$, with the gate centre and slope rescaled with the bandwidth. In square brackets, the number of orderings (of 24) whose mean threshold lies inside the range of the window sweep; where it is small the gap is measured against the clipped end of the curve. Top three settings: the declared band-block and category-block orderings. Others: orderings in which blocks of 128 prompts cycle through the regimes.}
\label{tab:bandwidth}
\resizebox{\textwidth}{!}{%
\begin{tabular}{lrrrrrr}
\toprule
Setting & Rule & $h/4$ & $h/2$ & $h$ & $2h$ & DtACI (no kernel) \\
\midrule
RealToxicityPrompts, GPT-2: band blocks & PromptShift-CRC & +0.0011 (0.0005) [21] & +0.0030 (0.0014) [15] & -0.0010 (0.0010) [19] & -0.0006 (0.0008) [17] & -0.0030 (0.0015) [24] \\
 & Kernel + update & -0.0074 (0.0023) [0] & -0.0024 (0.0014) [8] & -0.0027 (0.0012) [15] & -0.0026 (0.0012) [15] &  \\
RealToxicityPrompts, Qwen2.5-0.5B: band blocks & PromptShift-CRC & +0.0003 (0.0003) [24] & -0.0002 (0.0009) [20] & -0.0003 (0.0007) [20] & -0.0003 (0.0010) [15] & -0.0019 (0.0015) [24] \\
 & Kernel + update & -0.0031 (0.0025) [0] & -0.0002 (0.0012) [4] & -0.0004 (0.0011) [10] & -0.0008 (0.0012) [15] &  \\
MMLU, $\alpha=0.1$: category blocks & PromptShift-CRC & +0.0011 (0.0009) [22] & -0.0000 (0.0007) [23] & -0.0007 (0.0006) [17] & -0.0006 (0.0007) [14] & -0.0084 (0.0012) [0] \\
 & Kernel + update & +0.0058 (0.0015) [0] & -0.0001 (0.0008) [20] & -0.0007 (0.0006) [14] & -0.0010 (0.0005) [11] &  \\
RealToxicityPrompts, GPT-2: recurring regimes & PromptShift-CRC & -0.0000 (0.0001) [24] & -0.0001 (0.0001) [24] & -0.0001 (0.0000) [24] & +0.0000 (0.0000) [24] & -0.0056 (0.0003) [24] \\
 & Kernel + update & -0.0011 (0.0001) [10] & -0.0001 (0.0000) [5] & +0.0000 (0.0000) [20] & +0.0000 (0.0000) [22] &  \\
RealToxicityPrompts, Qwen2.5-0.5B: recurring regimes & PromptShift-CRC & -0.0001 (0.0000) [24] & -0.0001 (0.0001) [24] & +0.0000 (0.0001) [24] & +0.0001 (0.0000) [24] & -0.0061 (0.0003) [24] \\
 & Kernel + update & -0.0018 (0.0002) [18] & -0.0000 (0.0000) [7] & -0.0000 (0.0000) [20] & +0.0000 (0.0000) [22] &  \\
MMLU, $\alpha=0.1$: recurring regimes & PromptShift-CRC & -0.0002 (0.0001) [24] & +0.0001 (0.0001) [20] & -0.0002 (0.0001) [16] & -0.0004 (0.0001) [16] & -0.0062 (0.0003) [0] \\
 & Kernel + update & -0.0033 (0.0002) [0] & -0.0006 (0.0001) [21] & -0.0005 (0.0001) [14] & -0.0004 (0.0001) [16] &  \\
MMLU, $\alpha=0.2$: recurring regimes & PromptShift-CRC & -0.0002 (0.0001) [24] & -0.0002 (0.0001) [24] & -0.0002 (0.0001) [16] & -0.0003 (0.0001) [18] & -0.0043 (0.0004) [2] \\
 & Kernel + update & -0.0033 (0.0001) [0] & -0.0007 (0.0001) [21] & -0.0004 (0.0001) [16] & -0.0004 (0.0001) [15] &  \\
MMLU, $\alpha=0.3$: recurring regimes & PromptShift-CRC & -0.0001 (0.0001) [24] & -0.0001 (0.0001) [24] & -0.0000 (0.0001) [23] & +0.0001 (0.0001) [20] & -0.0058 (0.0006) [4] \\
 & Kernel + update & -0.0029 (0.0003) [0] & -0.0002 (0.0001) [23] & -0.0002 (0.0001) [17] & -0.0001 (0.0001) [16] &  \\
\bottomrule
\end{tabular}
}
\end{table}

\begin{table}[H]
\centering
\small
\caption{Does the pre-gate drift predict violations? For each stream the steps after the calibration prefix are split by whether the pre-gate drift is in the top decile of that stream. Columns: violation rate in the top decile minus the rate elsewhere, AUC of the drift for the violation indicator (mean and standard error over orderings, or over seeds for the synthetic streams), and the Spearman correlation across shifts between the jump in drift at the boundary and the post-shift violation, with a 95\% bootstrap interval. Left: violations of PromptShift-CRC (the primary outcome). Right: violations of the matched baseline (robustness).}
\label{tab:trace}
\resizebox{\textwidth}{!}{%
\begin{tabular}{lrrrrrr}
\toprule
Stream & PS: top decile $-$ rest & PS: AUC & PS: jump vs post-shift & Matched: top decile $-$ rest & Matched: AUC & Matched: jump vs post-shift \\
\midrule
RealToxicityPrompts, GPT-2 (bands) & -0.044 (0.002) & 0.433 (0.002) & -0.63 [-0.77, -0.47] & -0.042 (0.002) & 0.431 (0.002) & -0.38 [-0.55, -0.18] \\
RealToxicityPrompts, Qwen2.5-0.5B (bands) & -0.027 (0.002) & 0.437 (0.002) & -0.66 [-0.78, -0.51] & -0.029 (0.002) & 0.435 (0.002) & -0.24 [-0.46, -0.00] \\
MMLU (categories) & +0.024 (0.003) & 0.515 (0.002) & -0.25 [-0.45, -0.02] & +0.025 (0.004) & 0.517 (0.002) & -0.15 [-0.36, +0.08] \\
synthetic: pilot & +0.071 (0.005) & 0.587 (0.004) & +0.23 [+0.01, +0.42] & +0.074 (0.005) & 0.588 (0.004) & +0.37 [+0.16, +0.53] \\
synthetic: ablation stream & +0.050 (0.004) & 0.553 (0.004) & +0.83 [+0.76, +0.87] & +0.054 (0.004) & 0.558 (0.004) & +0.83 [+0.74, +0.88] \\
synthetic: QA topic shift & +0.057 (0.005) & 0.565 (0.004) & -0.37 [-0.54, -0.16] & +0.055 (0.004) & 0.560 (0.004) & -0.34 [-0.52, -0.14] \\
synthetic: summarization shift & +0.047 (0.004) & 0.562 (0.004) & -0.39 [-0.56, -0.20] & +0.045 (0.004) & 0.556 (0.004) & -0.36 [-0.54, -0.17] \\
synthetic: safety-intent shift & +0.051 (0.005) & 0.553 (0.004) & -0.20 [-0.38, +0.00] & +0.050 (0.006) & 0.547 (0.004) & -0.15 [-0.36, +0.06] \\
\bottomrule
\end{tabular}
}
\end{table}

\begin{table}[H]
\centering
\small
\caption{Settings used in every experiment. $W$ is the full window, ``recent'' the recency window, $h$ the kernel bandwidth, gate centre $c$ and slope $\lambda$ as in Section~\ref{sec:method}, $\gamma$ the step of the online update, and ``pool'' the length of the history used by the weighted rules. All use $\alpha=0.1$ except the rows that give another level. The synthetic constants are those of the earlier version of this work and were not changed; the real-stream constants follow rules fixed in advance. ``Bound'' is the Theorem~\ref{thm:longrun} bound on the distance between coverage and $0.90$. For the proxy streams the window lengths follow the rule in the released script (a fraction of the stream length with caps).}
\label{tab:hyper}
\resizebox{\textwidth}{!}{%
\begin{tabular}{lrrrrrrrrr}
\toprule
Experiment & Prefix & $W$ & Recent & $h$ & Gate $c$ & Slope $\lambda$ & $\gamma$ & Pool & Bound \\
\midrule
Synthetic pilot & 300 & 300 & 160 & 1.8 & 2.0 & 2.0 & 0.025 & 3W & 0.018 \\
Synthetic ablation stream & 300 & 300 & 160 & 1.85 & 2.05 & 2.2 & 0.02 & 3W & 0.022 \\
Synthetic prototype tasks (3) & 300 & 280 & 150 & 1.85 & 2.05 & 2.2 & 0.02 & 3W & 0.024 \\
MMLU & 256 & 256 & 128 & $h$ (median) & $1.1h$ & $2.5/h$ & 0.02 & 3W & 0.020 \\
RealToxicityPrompts (both generators) & 256 & 256 & 128 & $h$ (median) & $1.1h$ & $2.5/h$ & 0.02 & 3W & 0.012 \\
MMLU, $\alpha=0.2$ and $0.3$ & 256 & 256 & 128 & $h$ (median) & $1.1h$ & $2.5/h$ & 0.02 & 3W & 0.018 and 0.016 \\
RealToxicityPrompts (GPT-2), $\alpha=0.05$ and $0.2$ & 256 & 256 & 128 & $h$ (median) & $1.1h$ & $2.5/h$ & 0.02 & 3W & 0.013 and 0.011 \\
TruthfulQA & 100 & 100 & 50 & $h$ (median) & $1.1h$ & $2.5/h$ & 0.03 & 3W & 0.043 \\
Proxy: TruthfulQA fields & rule & rule & rule & 1.25 & 1.35 & 2.5 & 0.02 & 3W & 0.070 \\
Proxy: SummEval & rule & rule & rule & 0.75 & 0.85 & 4.0 & 0.018 & 3W & 0.040 \\
Proxy: LongHalQA & rule & rule & rule & 0.55 & 0.55 & 5.0 & 0.018 & 3W & 0.046 \\
Proxy: RealToxicityPrompts fields & rule & rule & rule & 0.9 & 0.95 & 3.0 & 0.015 & 3W & 0.017 \\
\bottomrule
\end{tabular}
}
\end{table}

\end{document}